\documentclass{article}
\usepackage[top=1in, bottom=1in, left=1in, right=1in]{geometry}
\usepackage[utf8]{inputenc}

\usepackage{overpic}
\usepackage{amsmath}
\usepackage{amstext}
\usepackage{amsfonts}
\usepackage[font=small,labelfont=bf]{caption}
\usepackage{dsfont}
\usepackage{graphicx}
\usepackage{algorithm}
\usepackage{algorithmic}
\usepackage{scalerel}
\usepackage{subcaption}
\graphicspath{ {./figs/} }

\usepackage[font=small,labelfont=bf]{caption}
\usepackage[backend=bibtex,style=numeric,sorting=none]{biblatex}
\usepackage{tikz}
\usetikzlibrary{positioning, fit, arrows.meta, shapes}

\usepackage{subeqnar}

\usepackage{amssymb,amsthm}
\usepackage{bm}
\usepackage{physics}
\usepackage{mathtools}
\usepackage{hyperref}
\usepackage{geometry}

\newtheorem{prop}{Proposition}
\newtheorem{thm}{Theorem}
\newtheorem{remark}{Remark}

\title{Data assimilation and discrepancy modeling with\\ shallow recurrent decoders
}
\author{Yuxuan Bao$^*$, J. Nathan Kutz$^{*,\dag}$\\[.1in]$^*$Department of Applied Mathematics, University of Washington, Seattle, WA\\
$^\dag$Department of Electrical and Computer Engineering, University of Washington, Seattle, WA }

\begin{document}

\maketitle

\begin{abstract}
The requirements of modern sensing are rapidly evolving, driven by increasing demands for data efficiency, real-time processing, and deployment under limited sensing coverage.  Complex physical systems are often characterized through the integration of a limited number of point sensors in combination with scientific computations which approximate the dominant, full-state dynamics. Simulation models, however, inevitably neglect small-scale or hidden processes, are sensitive to perturbations, or oversimplify parameter correlations, leading to reconstructions that often diverge from the reality measured by sensors.   This creates a critical need for data assimilation, the process of integrating observational data with predictive simulation models to produce coherent and accurate estimates of the full state of complex physical systems. We propose a machine learning framework for Data Assimilation with a SHallow REcurrent Decoder (DA-SHRED) which bridges the simulation-to-real (SIM2REAL) gap between computational modeling and experimental sensor data. For real-world physics systems modeling high-dimensional spatiotemporal fields, where the full state cannot be directly observed and must be inferred from sparse sensor measurements, we leverage the latent space learned from a reduced simulation model via SHRED, and update these latent variables using real sensor data to accurately reconstruct the full system state. Furthermore, our algorithm incorporates a {\em sparse identification of nonlinear dynamics} (SINDy) based regression model in the latent space to identify functionals corresponding to missing dynamics in the simulation model. We demonstrate that DA-SHRED successfully closes the SIM2REAL gap and additionally recovers missing dynamics in highly complex systems, demonstrating that the combination of efficient temporal encoding and physics-informed correction enables robust data assimilation under sparse sensing constraints.
\end{abstract}

\section{Introduction}

%
Data-driven science and engineering is being revolutionized by advancements in machine learning and AI algorithms~\cite{brunton2022data}.   Leveraging sensor measurements, often in combination with scientific computation proxies, such algorithms aim to learn effective models for a diversity of downstream tasks, including reconstruction, forecasting, and control in challenging environments that include noisy measurements and or parametric variability.    A grand challenge in the deployment of such algorithms is the fact that many physical systems are not amenable to full state measurements, but rather only discrete point sensor measurements at prescribed and limited locations.  In fact, the only knowledge of the full state space is typically approximated by simulations of the underlying governing equations which are often given by {\em partial differential equations} (PDEs).  For example, our knowledge of the full dynamics of nuclear reactors, plasma physics, rocket engines, and many complex flow fields has only been estimated and/or constructed by simulation proxies.  Some nuclear reactors, for instance, are modeled by up to 20 coupled PDEs that detail the complex interactions between the fluid dynamics, thermodynamics, ion concentrations, etc~\cite{riva2024robust}.  None of the 20 fields have been measured in practice in deployed reactors.  Rather, in reality only one or two of the fields (e.g. temperature, pressure) can be measured at discrete point sensor locations on the walls of the reactor.  This presents a significant modeling challenge for closing the {\em simulation-to-reality} (SIM2REAL) gap~\cite{hofer2021sim2real,kadian2020sim2real}, especially as the PDEs we simulate often poorly approximate the real physics of such complex systems.   Using the recently developed {\em SHallow REcurrent Decoder} (SHRED) architecture~\cite{williams2023sensing,tomasetto2025reduced,gao2025sparse}, we demonstrate a data-driven method for (i) updating a SHRED model to reality when trained only on simulations and (ii) additionally learning the missing physics of the simulation proxy.  Our {\em data assimilation SHRED} (DA-SHRED) thus provides an effective algorithm for closing the SIM2REAL gap in many complex systems, as demonstrated in the challenging examples presented here which include applications in rocket engines, chemical reactors and turbulent flows.

\begin{figure}[t]
\centering
\vspace*{-.5in}
\includegraphics[scale=0.55]{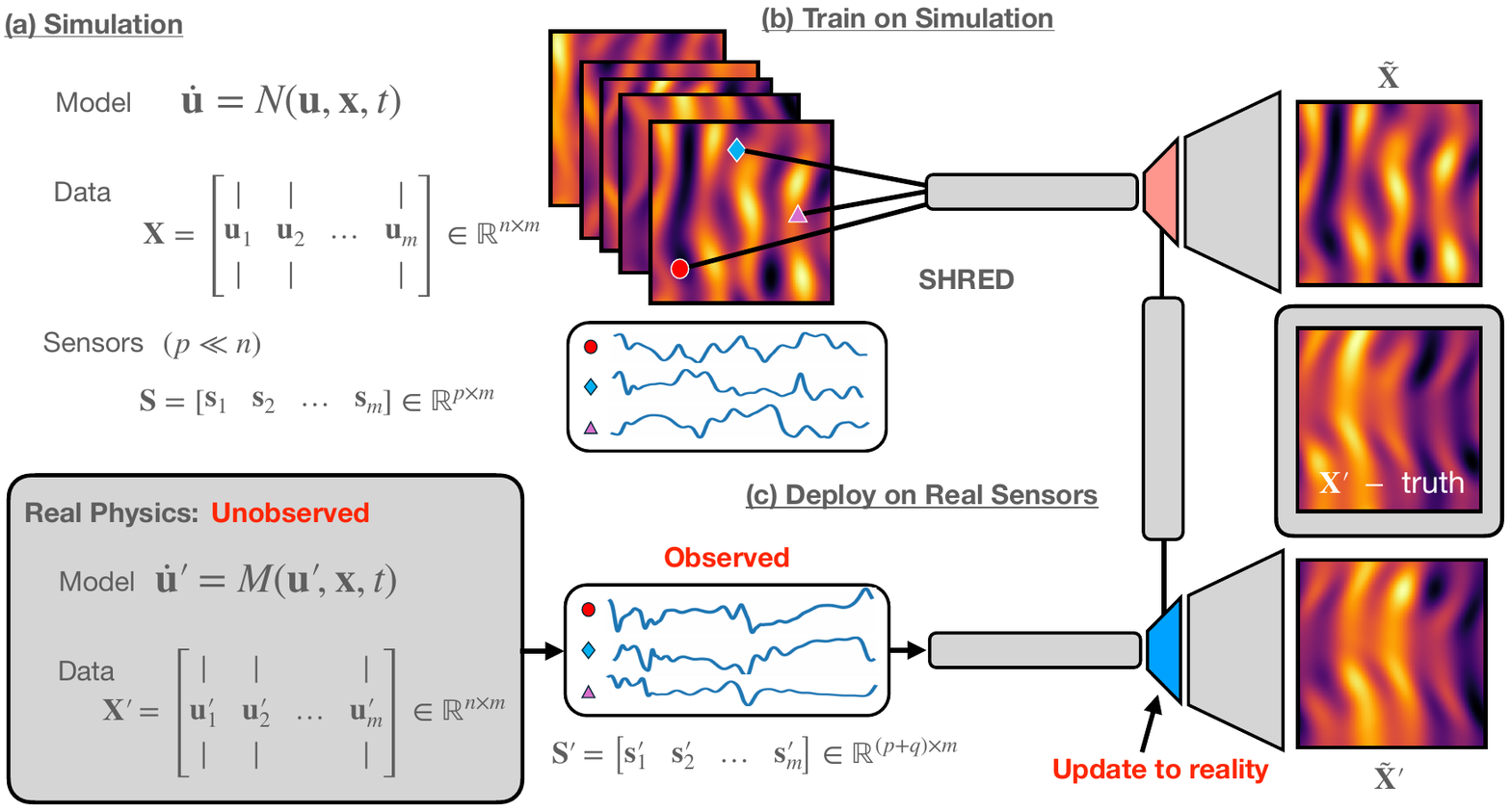}
\vspace*{-.5in}
\caption{The algorithmic structure of the DA-SHRED structure. With real physics unobserved, the model exploits the latent space trained on known simulation data using traditional SHRED network, and deploys it on real sensor data to close the discrepancy.  It is assumed that the full state-space of the real physics is never observed in practice.  The variables, data and models are summarized in the figure table~\ref{fig:table}.
\label{fig:overview}}
\end{figure}

%
Data assimilation has become the leading method for closing the SIM2REAL gap across science and engineering~\cite{law2015data}, with extensive theoretical and computational developments over the past two decades~\cite{evensen2009data, reich2015probabilistic}. Based upon Kalman filtering, data assimilation is potentially one of the most useful and broadly deployed data-driven modeling technique available today as we are rarely without access to some underlying governing equations or without experimental measurements.   By assuming that both the model used and the measurements acquired have known error distributions, ensemble Kalman filtering methods~\cite{evensen2003ensemble,houtekamer2005ensemble} can be used to generate optimal statistical predictions.   Weather forecasting has been revolutionized by data assimilation, with state-of-the-art methods like the 4DVAR architecture~\cite{lorenc2003modelling,fisher2001developments} providing a remarkable improvement in our modern forecasting capabilities.  Data assimilation, however, does not typically use the SIM2REAL mismatch to learn updates to the underlying model.  Thus the SIM2REAL gap often persists.  Discrepancy modeling for dynamic systems~\cite{kaheman2019learning,ebers2024discrepancy,levine2022framework} attempts to address this issue by using the SIM2REAL gap to propose updates to the underlying model in order to correct the physics towards reality.  Discrepancy modeling is thus driven by sensor measurements which are a direct assessment of reality modulo noise or bias in the measurements.  A challenge for both data assimilation and discrepancy modeling is that simulations allow access to all variables at all spatial and temporal points of a computational mesh or grid.  In contrast, discrete temporal measurements of reality are typically only acquired at a limited number of spatial positions, and typically only a subset of the variables are actually observed.  Thus the goal of updating the underlying PDE (governing equations) is exceptionally challenging, with only limited methods proposed thus far.

The DA-SHRED algorithm proposed here, illustrated in Fig.~\ref{fig:overview}, is based upon the SHRED architecture~\cite{williams2023sensing,tomasetto2025reduced,gao2025sparse} which leverages three key mathematical concepts:  (i) the separation of variables, (ii) Takens embedding theorem, and (iii) a decoding only strategy.  Separation of variables is the foundation of many analytic solutions techniques for solving PDEs and it has been deployed as the infrastructure for numerical time stepping methods for PDEs.  Takens' embedding theorem states that the time-history information at a single measurement location can provide a diffeomorphic representation of the entire state-space of a dynamical system.  With modern machine learning, this diffeomorphism can be learned with training data.  Finally, the decoding only strategy avoids the ill-conditioned learning of inverse operators required of standard autoencoder/decoder pairs.  In combination, they provide SHRED with a robust framework for sensing~\cite{williams2023sensing}, model reduction~\cite{tomasetto2025reduced}, and physics discovery~\cite{gao2025sparse}.  The innovations presented here leverage the SHRED architecture to not only close the SIM2REAL gap, but to also learn the missing physics in the original simulations.  SHRED is trained initially on the simulation data, and then updated by the acquisition of temporal sensor data deployed on the real physical system.  This gives an optimization pathway for updating the weights of the DA-SHRED model trained on simulation to minimize the error between the sensors observations and the SHRED predictions at the sensor locations.  In addition to minimizing the SIM2REAL error, the updated DA-SHRED model can then be used to estimate the missing physics terms which are responsible for the error in the SIM2REAL gap.  Specifically, the latent space of the DA-SHRED architecture is embedded with a {\em sparse identification of nonlinear dynamics} (SINDy)~\cite{brunton2016discovering,rudy2017data} loss functional in order to extract the parsimonious missing physics from a library of potential candidate terms.  Thus we propose a 2-stage scheme for both data assimilation and discrepancy modeling.  The success of the method is demonstrated on a number of challenging models that include the 2D Kuramoto-Sivanshinsky equations, the 2D Kolmogorov flow, the 2D Gray-Scott reaction-diffusion equations, and a rotating detonation engine example.  In each case, the DA-SHRED is shown to be effective in closing the SIM2REAL gap and learning the missing physics.  The algorithm is further robust to noise, requires only minimal training data, and can be trained in a compressive framework which allows for efficient laptop level computing.

\begin{figure}[t]
\centering
\vspace*{-1in}
\includegraphics[scale=0.55]{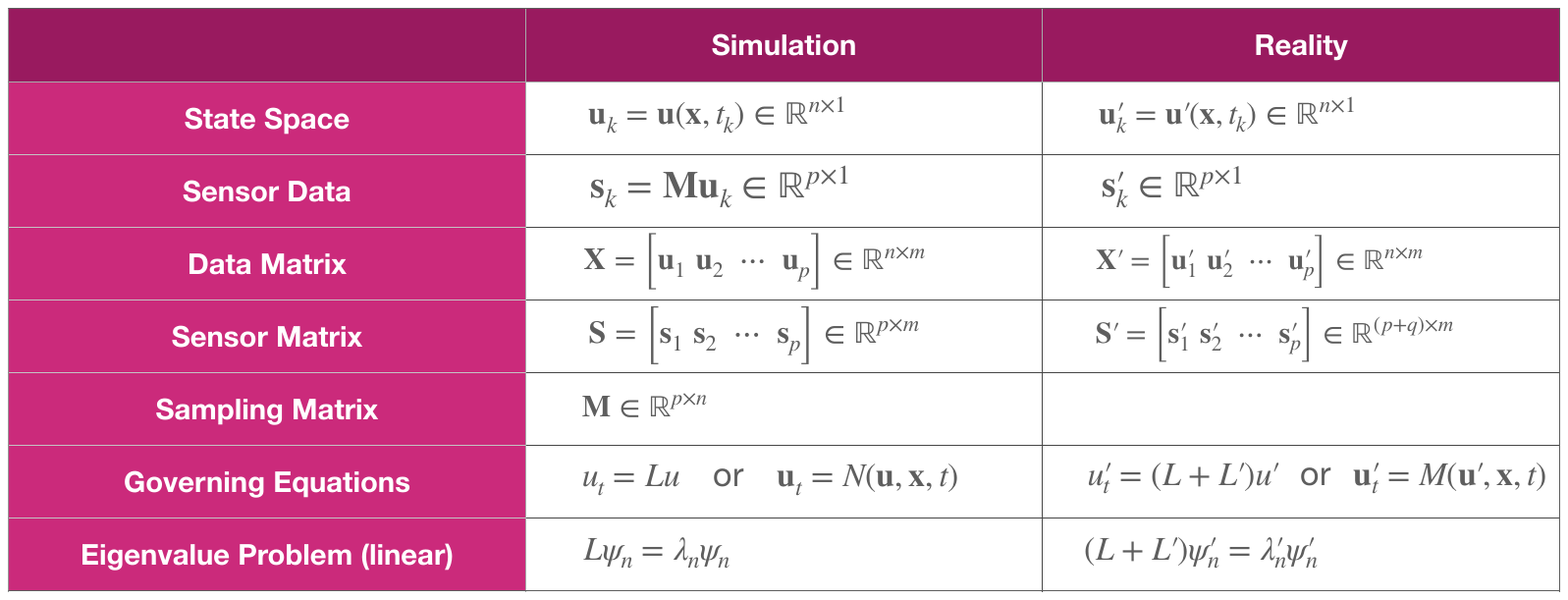}
\vspace*{-1in}
\caption{Summary of variables, data and models used in the DA-SHRED formulation.  The state space is of dimension $n$, there are $m$ snapshots of temporal measurements using $p$ sensors for SHRED training with an additional $q$ sensors deployed in reality. 
\label{fig:table}}
\end{figure}

\section{The Data-Assimilation SHRED (DA-SHRED) Architecture}


The overall goal is to close the SIM2REAL gap using the underlying SHRED architecture~\cite{williams2023sensing}.  Specifically, can we formally show that this can be done under suitable assumptions outlined here.  Thus given simulation data ${\bf X}$ for training a SHRED model, can we estimate the full state-space of real data ${\bf X}'$ given only limited point sensor measurements of reality ${\bf S}'$.  Or more formally:
\begin{eqnarray*}
  && \mbox{\bf input:} \,\,\, {\bf X}, {\bf S}' \\
  && \mbox{\bf output:} \,\,\,  {\bf X}'
\end{eqnarray*}
Or more precisely, we want a model mapping reality measurements to reality full-state estimates
\begin{equation}
  {\bf X}' = {\cal F}_{\theta} ({\bf S}')
\end{equation}
where we will learn a neural network model ${\cal F}$ with weights $\theta$ from training on simulation data ${\bf X}$ and updating it to match reality at the $p+q$ sensors deployed in reality. Additionally, we can also perform the following discrepancy modeling task with the DA-SHRED architecture:
\begin{eqnarray*}
  && \mbox{\bf input:} \,\,\, {\bf X}, {\bf S}' \\
  && \mbox{\bf output:} \,\,\,  {\bf X}', {{L}'}
\end{eqnarray*}
Thus we not only update the simulation  to reality by learning ${\cal F}_\theta$, but we also explicitly learn the missing physics term ${L}'$ using the SINDy method~\cite{brunton2016discovering,ebers2024discrepancy}.

\subsection{Linear, constant coefficient PDEs}

For this case, we can explicitly compute how the DA-SHRED architecture is able to update the model to data. The dynamics of the linear PDE model approximating reality are given by the governing equations
\begin{equation}
    u_t = Lu
\end{equation}
with some specified boundary conditions.  Note that we don't actually know the evolution equations for $u'$.  For constant coefficient operator $L$ and the corresponding reality model $L+{L}'$ (See Fig.~\ref{fig:table}), solutions can be easily calculated via separation of variables (approximated with $N$ modes) so that
\begin{equation}
    u = \sum_{n=1}^N a_n \psi_n \exp(\lambda_n t)
    \label{eq:basis_sim}
\end{equation}
and
\begin{equation}
    u' = \sum_{n=1}^N a'_n \psi'_n \exp(\lambda'_n t).
    \label{eq:basis_reality}
\end{equation}
The problem:  we don't actually know the triplet $a'_n, \psi'_n, \lambda'_n$ since we don't know ${L}'$. It was shown in~\cite{tomasetto2025reduced} that we can solve the {\em sensor boundary value problem} (SBVP) using point measurements and time trajectories at arbitrary sensor locations to uniquely determine the $a_n$ in  (\ref{eq:basis_sim}). This is instead of of the standard {\em initial value problem} (IBVP) which requires the entire state-space be known at a given instance of time for solving the PDE and determining the $a_n$.

Assuming that the solution of (\ref{eq:basis_reality}) can be approximated by the basis of (\ref{eq:basis_sim}), the solution for $u'$ is instead given by:
\begin{equation}
    u' = \sum_{n=1}^N a'_n \psi_n \exp(\lambda'_n t).
    \label{eq:basis_reality}
\end{equation}
Then we would use ${\bf S}'$ to update the simulation values to reality by computing $a'_n$ and $\lambda'_n$ from these measurements.  As shown  in~\cite{tomasetto2025reduced}, this can be done with simple linear algebra. Thus the SIM2REAL gap can be shown to be explicitly closed for linear, constant coefficient PDEs provided a common basis is sufficient for the approximation.

This gives a theoretical construct for understanding how the DA-SHRED architecture works.  Additional complexity arises for nonlinear PDEs.  If they are weakly nonlinear PDEs, then under mild assumptions, the above arguments hold perturbatively.  For strongly nonlinear dynamics, DA-SHRED is enacted computationally to show that the arguments above hold.  This will be shown in the four challenging examples of this paper.  Importantly, the underlying assumption in all these cases is that the learned decoder bases of the simulation remain a good approximation basis for real physics.  Otherwise, updating the bases to reality also needs to be considered.

As a final point, we highlight how the missing physics is found.  This is done by using a SINDy regression to model the missing physics~\cite{brunton2016discovering,kaheman2019learning,ebers2024discrepancy}:
\begin{equation}
    {L}' = \Theta \xi
    \label{eq:Lprime}
\end{equation}
where $\Theta$ is the library of candidate functions and $\xi$ is the loading vector which is required to be sparse.  We should be easily able to do this as this is directly related to what has been done previously for learning discrepancy terms~\cite{kaheman2019learning,ebers2024discrepancy}. 

As noted previously, if the basis representation shifts significantly between the simulation model and reality, then the decoder must be updated to more accurately represent the reality basis modes.  This remains challenging to do as there are no measurements for the reality modes, only a limited number of point measurements.  Thus the DA-SHRED model has no capability of modeling this out-of-distribution data.  This is beyond the scope of the current work.  However, in the models considered here, the assumption of using the basis learned from the simulation data appears to be sufficient to give highly accurate models.

\subsection{Port-Hamiltonian Systems}

Port-Hamiltonian systems (PHS) are a framework for modeling physical dynamics that exhibit energy conservation, interconnection and dissipation structures. Originating from classical Hamiltonian mechanics, PHS models have been formalized largely in the past two decades as an extension for modeling interactions, dissipation and enabling control inputs.

The input-state-output PHS, with no constraint on the state space variables, is typically represented in terms of the following equation\cite{van2006port}:

\[
\dot{x} = (J(x) - R(x)) \frac{\partial H}{\partial x} + g(x)u
\]
And the output equation is:
\[
y = g^T(x) \frac{\partial H}{\partial x}
\]
where $u$ and $y$  are input–output pairs, $H(x)$ is a Hamiltonian function,  $J(x)$ is skew-symmetric (\( J^T(x) = -J(x) \)), \( R(x) \) is a symmetric positive semi-definite dissipation matrix (\( R(x) = R^T(x) \geq 0 \)) that specifies a resistive structure, and \( g \) is the input matrix.

\begin{thm}[Persistence of PHS form under natural perturbations]\label{thm:main}
Let \eqref{eq:phs_basic} be a PHS on a simply connected domain $\mathcal{X}$. Consider perturbations $\Delta J,\Delta R,\Delta G$ (matrix fields) and a 1-form perturbation $\delta$ so that the perturbed dynamics are
\[
\dot x = \big((J+\Delta J)-(R+\Delta R)\big)\big(\nabla H + \delta\big) + (G+\Delta G)u.
\]
Assume:
\begin{enumerate}
  \item $\Delta J$ is skew-symmetric on $\mathcal{X}$,
  \item $\Delta R$ is symmetric and small enough on compact subsets to keep $R+\Delta R\succeq 0$,
  \item $\delta$ is closed: $d\delta = 0$ on $\mathcal{X}$.
\end{enumerate}
Then there exists $\widetilde H = H + \varepsilon$ with $\mathrm{d}\varepsilon=\delta$ and
\[
\widetilde J := J+\Delta J,\qquad \widetilde R := R+\Delta R,\qquad \widetilde G := G+\Delta G
\]
such that the perturbed system can be written globally in PHS form
\[
\dot x = (\widetilde J - \widetilde R)\nabla \widetilde H + \widetilde G\, u.
\]
\end{thm}

In an eigenspace-preserving PHS, the evolution of the dynamics is maintained within a specific decomposition of the state space. Such invariant could be useful in multiple fields including structured mechanical\cite{warsewa2021port}, electrical\cite{gernandt2021port}, or fluid \cite{rashad2021port} systems where distinct eigenmodes are expected under the PHS structure. In practice, the preservation of eigenspaces in port-Hamiltonian systems  often involves the careful design of $J$ and $R$ matrices, such that the transformation follows the spectral properties of the Hamiltonian \cite{van2016interconnections}.  Importantly, the PHS is a broad class of systems for which the decocder space is preserved and invariant.  Thus training DA-SHRED on PHS allows the decoder to discover an approximation of the invaraiant space.

\section{Results}

DA-SHRED provides an effective algorithm for closing the SIM2REAL gap in many complex physical systems, as demonstrated in the challenging examples presented here which include applications in rocket engines, chemical reactors and turbulent flows.  These examples generalize the underlying theory presented in the preceding section and establish the effectiveness of the algorithm on highly nontrivial spatio-temporal systems.




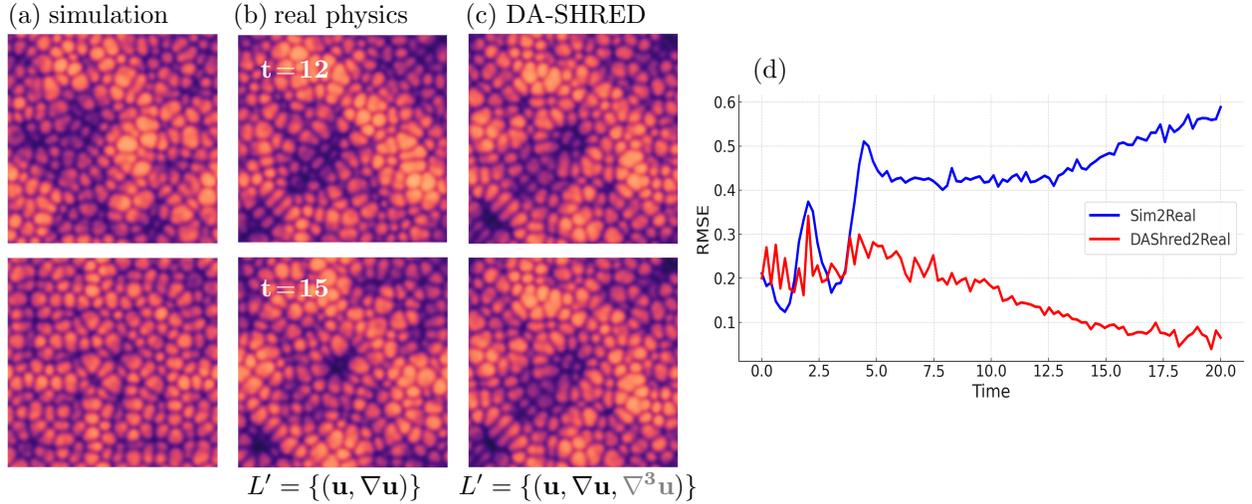
\begin{figure}[t]
    \begin{minipage}{0.55\linewidth}
  \begin{subfigure}{\linewidth}
    \begin{overpic}[width=\linewidth]{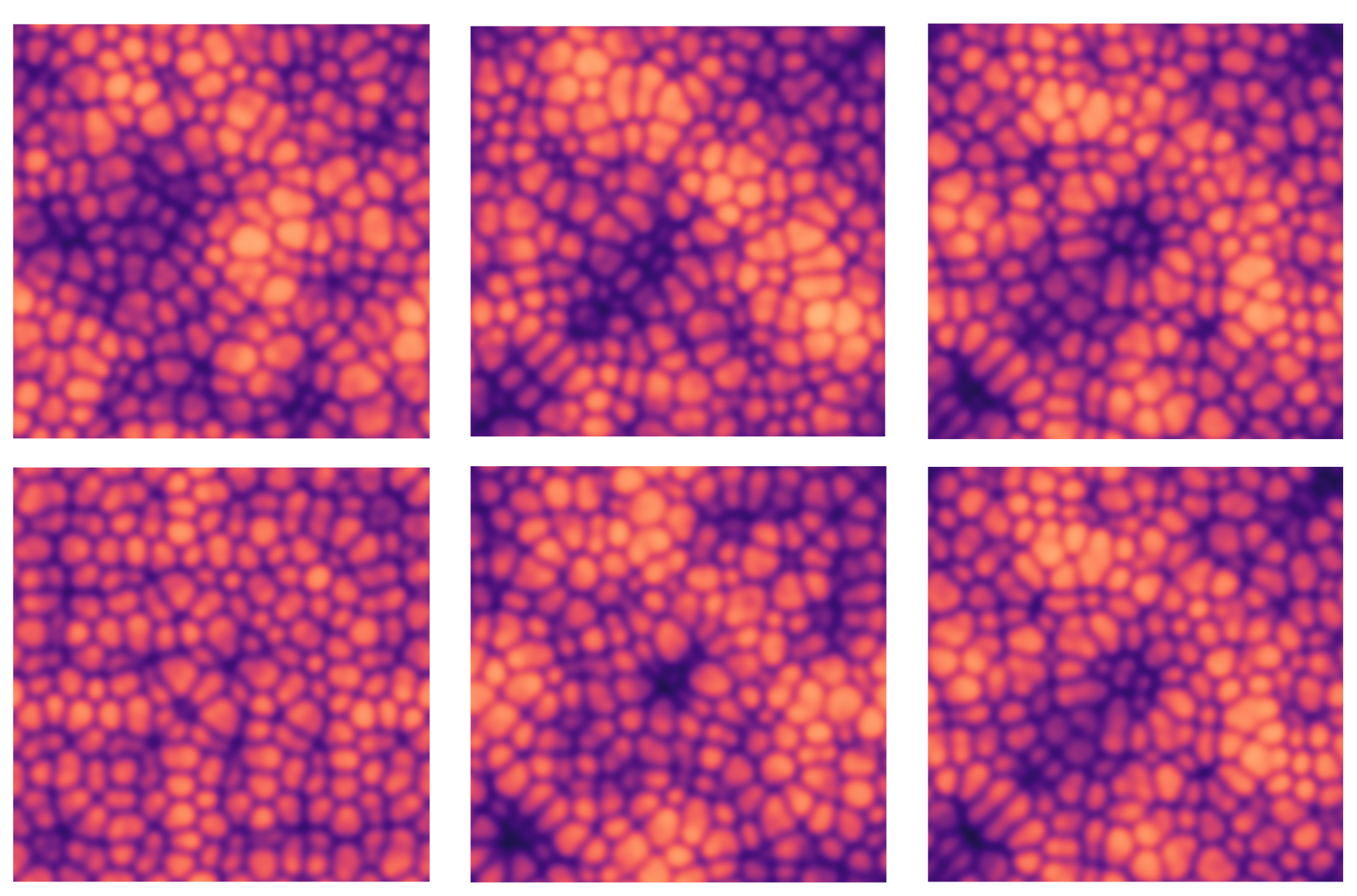}
    \put(7,66){simulation}
    \put(40,66){real physics}
    \put(74,66){DA-SHRED}
    \put(38,58){\color{white} $\bf t\!=\!12$}
    \put(38,26){\color{white} $\bf t\!=\!15$}
    \put(36,-3){\color{black} ${{L}'} = \{(\bf{u},\nabla \bf{u})\}$}
    \put(67,-3){\color{black} ${{L}'} = \{(\bf{u}, \nabla \bf{u}, {\color{gray}{\nabla ^3 \bf{u}}})\}$}
    \put(1,66){(a)}
    \put(34,66){(b)}
    \put(68,66){(c)}
    \put(110,58){(d)}
    \end{overpic}
    \caption*{}\label{subfig:key-a}
  \end{subfigure}
  \end{minipage}
  \begin{minipage}{0.45\linewidth}
    \begin{subfigure}{\linewidth}
      \vstretch{1.15}{\includegraphics[width=\linewidth]{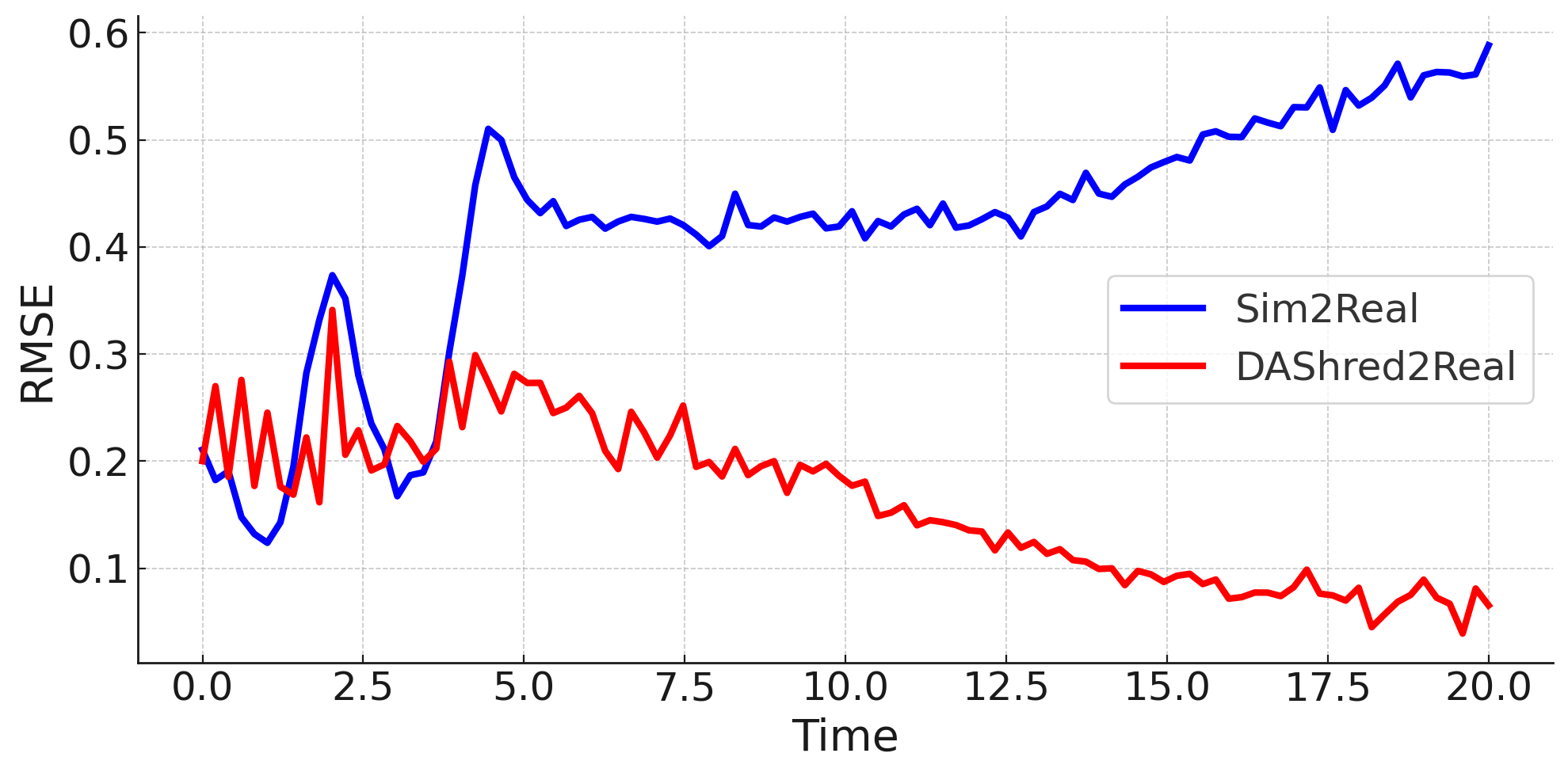}}
      \caption*{}\label{subfig:key-b}
    \end{subfigure}\hfill
  \end{minipage}
  \caption{The result of 2D damped Kuramoto-Sivashinsky equation. Figures on the same row are taken at the same timestep. The
left column (a) represents the undamped simulation model (2D KS equation without damping); the middle column (b)
represents unknown real physics (2D KS equation with damping); and the right column (c) represents state
space restored by DA-SHRED.  The error for the DA-SHRED correction to reality is shown in (d) where an order of magnitute improvement in accuracy is achieved within $T=20$.\label{fig:2DKS}}
\end{figure}

\subsection{2D damped Kuramoto–Sivashinsky equation}

The first example considered is the 2D version of the damped Kuramoto-Sivashinsky (2DKS) equation~\cite{jayaprakash1993universal,kalogirou2015depth}.  
The KS equation is one of the simplest canonical PDE models that produces spatio-temporal chaos.  It models a wide range of phenomena including instabilities of a laminar flame fronts, the dynamics of thin liquid films on inclined plane, and trapped-ion instabilities in a plasma.  The typical 2DKS, whose dynamics are illustrated in Fig.~\ref{fig:2DKS}, is given by the evolution equation
\begin{equation}
\bf{u}_t + \frac12 |\nabla \bf{u}|^2 + \nabla^2 \bf{u} + \nu\nabla^4 \bf{u} = 0 .
\label{eq:2DKS_sim}
\end{equation}
We assume this to be our model of reality. However, the true physics in our example will be given by a damped version of the 2DKS \cite{dehghan2019two} which is given by
\begin{equation}
\bf{u}_t^\prime  + \frac12 |\nabla \bf{u}^\prime|^2 + \nabla^2 \bf{u}^\prime + \nu\nabla^4 \bf{u}^\prime + (\bf{v}\cdot \nabla)\bf{u}^\prime - \gamma \bf{u}^\prime= 0
\label{eq:2DKS_real}
\end{equation}
with the additional damping term $(\bf{v}\cdot \nabla)\bf{u}^\prime - \gamma \bf{u}^\prime$. These additional terms represent the discrepancy from reality that must be learned in order to produce an effective model for prediction and characterization.  Thus DA-SHRED must first learn how to map from ${\bf u}$ to ${\bf u}'$ before eventually learning the additional damping terms.  Without loss of generality, we set the parameter $\nu = 1$. Figure \ref{fig:2DKS} exhibits the 2D evolution at times $t=12$ and $t=15$ on a domain of size $x\in[0,64]\times y\in[0,64]$.  The panels display the true physics (\ref{eq:2DKS_real}) (middle panels (b)) in comparison with the simulation of presumed reality (\ref{eq:2DKS_sim}) (left panels (a)) and the DA-SHRED improvement to reality (right panels (c)).  Importantly, not only does the DA-SHRED close the SIM2REAL gap, but it does so in almost real time as shown by the red error plot in the right panel (d) of Fig.~\ref{fig:2DKS}.  In contrast, the simulation model approximating reality continues to diverge from the true dynamics as shown by the blue error curve.  The behavior exhibited for the 2DKS is representative of all the other models considered here, i.e. the DA-SHRED architecture can be deployed to quickly update the model predictions and close the SIM2REAL gap.  Specifially, within a short evolution window, the DA-SHRED closes the SIM2REAL gap by approximately an order of magnitude in the error.  And as will be shown in the next section, the DA-SHRED architecture further allows for learning of the unknown discrepancy term of reality $L'$.  This will be discussed in the algorithms of the next section.

\subsection{2D Kolmogorov flow}

In the second example we consider the two-dimensional (2D) Kolmogorov flow, which is a canonical model in fluid dynamics.  Originally introduced by Arnold and Meshalkin in the 1960s to study hydrodynamic stability under periodic forcing~\cite{arnold1960, meshalkin1961investigation}. It is governed by the incompressible Navier--Stokes equations subject to a sinusoidal body force, typically expressed as $\mathbf{f}(x, y) = F \sin(k y)\,\hat{\mathbf{x}}$, which induces a laminar base flow that can lose stability as the Reynolds number increases. Despite its conceptual simplicity, the 2D Kolmogorov flow exhibits a rich spectrum of dynamical behaviors, including secondary instabilities, chaotic attractors, and quasi-periodic motion~\cite{platt1991investigation}. These features make it an ideal platform for investigating the mechanisms underlying turbulence transition and the emergence of coherent structures in two-dimensional flows~\cite{frisch1995turbulence, chandler2013invariant}. The system’s fully periodic boundary conditions and analytically tractable base state allow detailed numerical and theoretical studies of bifurcation sequences and energy transfer across scales~\cite{chandler2013invariant, lucas2014spatiotemporal}. More recent work has employed the Kolmogorov flow as a benchmark for reduced-order modeling, instability analysis, and prediction of extreme events such as rogue waves~\cite{farazmand2017reduced}.

In what follows, the 2D Kolmogorov flow is used as a fundamental testbed for understanding the performance of DA-SHRED. Our simulation model is a 2D Navier–Stokes equation with Kolmogorov forcing, which is given by
\begin{equation}
\begin{aligned}
\frac{\partial \bf{u}}{\partial t} + \bf{u} \cdot \nabla \bf{u} + \nabla p &= \nu \nabla^2 \bf{u} + F, 
\\
\nabla \cdot \bf{u} &= 0
\end{aligned}
\end{equation}
where $\bf{u}$ is a two-dimensional velocity field, with pressure $p$, forcing $F$ and viscosity $\nu = \frac{1}{R_e}$. Reformulated in terms of vorticity $\omega = \nabla \times \bf{u}$, the governing equation can be rewritten as~\cite{borue1996numerical}
\begin{equation}
\begin{aligned}
\frac{\partial \omega}{\partial t} + \bf{u} \cdot \nabla \omega &= \nu \nabla^2 \omega + f, 
\\
\nabla \bf{u} &= 0
\end{aligned}
\label{eq:2DKol_sim}
\end{equation}
This gives a model of idealized 2D Kolmogorov flow.  The true physics is  similarly assumed to be a damped version of the 2D Kolmogorov flow, given by
\begin{equation}
\begin{aligned}
\frac{\partial \omega^\prime}{\partial t} + \bf{u}^\prime \cdot \nabla \omega^\prime + \alpha \omega^\prime  &= \nu \nabla^2 \omega^\prime + f, 
\\
\nabla \bf{u}^\prime &= 0
\end{aligned}
\label{eq:2DKol_real}
\end{equation}
with $\nu = 0.01$, linear damping coefficient $\alpha = 0.12$, and $f = \mu \sin(y)$ where $\mu$ is the forcing amplitude.

Figure~\ref{fig:2DKol} exhibits the 2D Kolmogorov evolution of the $x$ and $y$ components of velocity (top and bottom panels respectively) on the domain $x\in[0,64]\times y\in[0,64]$ for times $t=30$ and $t=35$. As displayed previously, the panels show the true physics (\ref{eq:2DKol_real}) (middle panels (b)) in comparison with the simulation of presumed reality (\ref{eq:2DKol_sim}) (left panels (a)) and the DA-SHRED improvement to reality (right panels (c)).  As with the 2DKS, not only does the DA-SHRED close the SIM2REAL gap, but it does so in almost real time as shown by the error plot in the right panel (d) of Fig.~\ref{fig:2DKol}.  The error plots show the short-time and long-time decay of the error as the DA-SHRED continues to refine the solution approximation to reality.

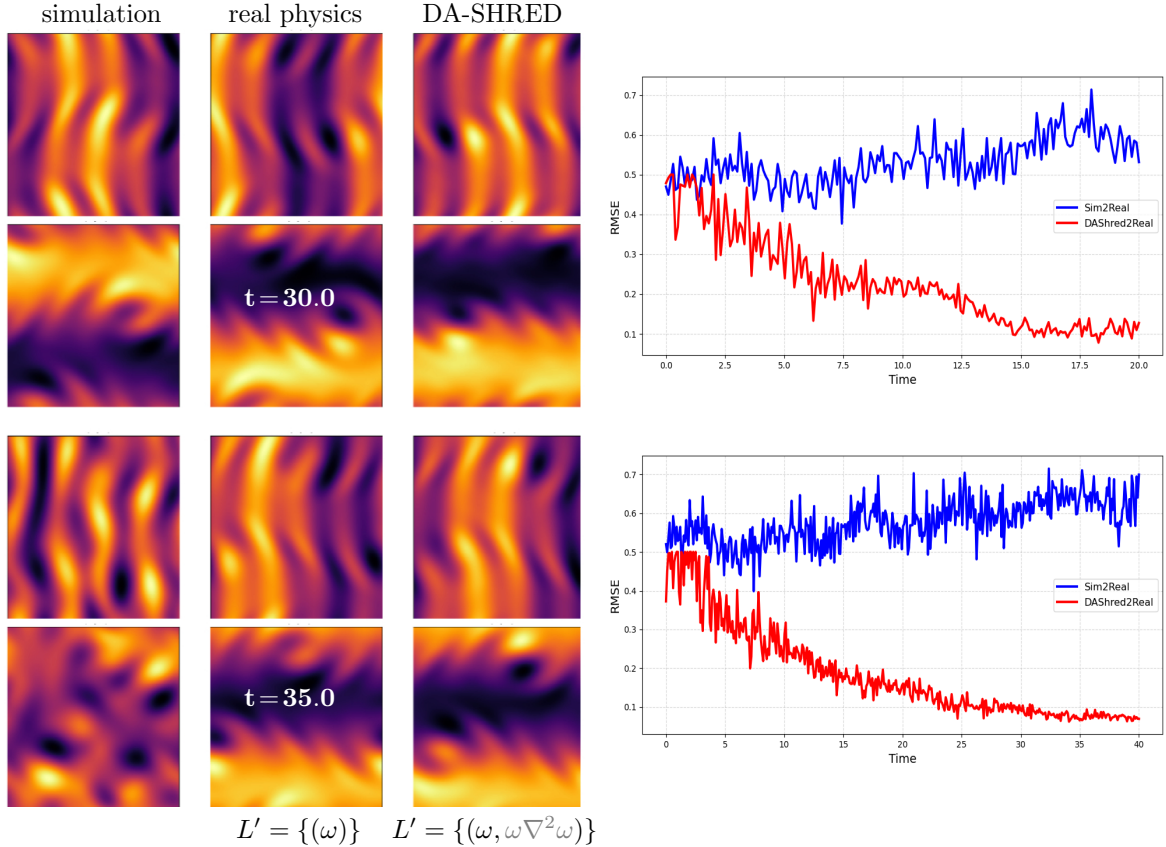
\begin{figure}[t]
\centering
\begin{minipage}{0.5\linewidth}
  \begin{subfigure}{\linewidth}
    \begin{overpic}[width=\linewidth]{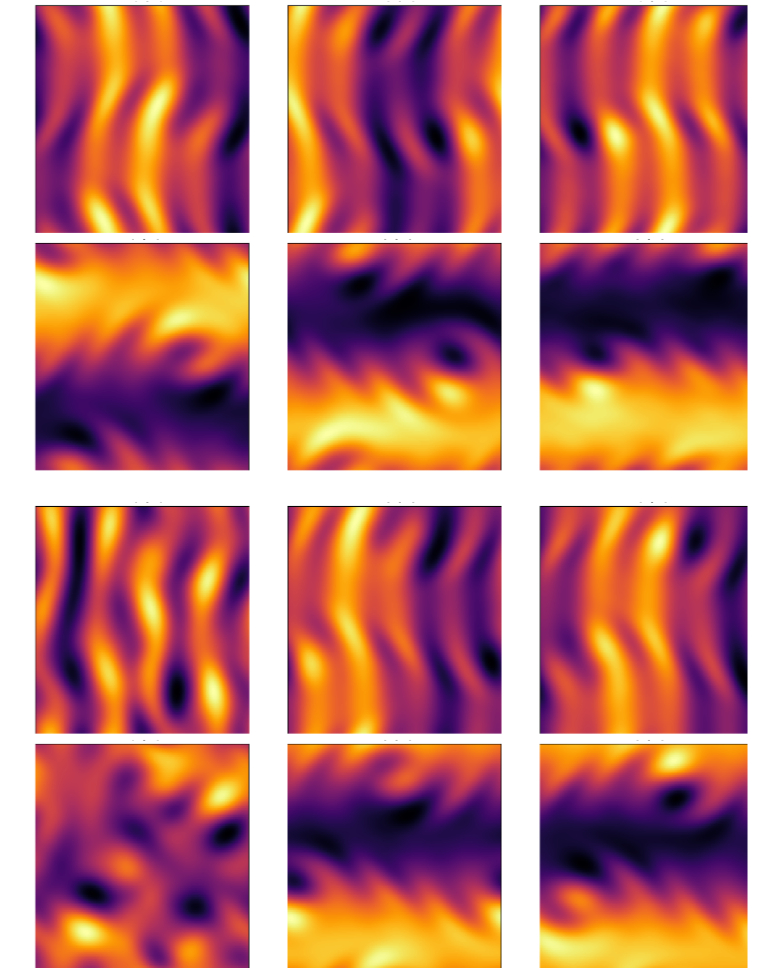}
    \put(8,101){simulation}
    \put(32,101){real physics}
    \put(57,101){DA-SHRED}
    \put(34,13){\color{white} $\bf t\!=\!35.0$}
    \put(34,64.5){\color{white} $\bf t\!=\!30.0$}
    \put(33,-4){\color{black} ${{L}'} = \{(\omega)\}$}
    \put(53,-4){\color{black} ${{L}'} = \{(\omega, {\color{gray}{\omega\nabla^2 \omega}} )\}$}
    \end{overpic}
    \caption*{}\label{subfig:key-a}
  \end{subfigure}
  \end{minipage}
\begin{minipage}{0.45\linewidth}
    \begin{subfigure}{\linewidth}
      \vstretch{1.1}{\includegraphics[width=\linewidth]{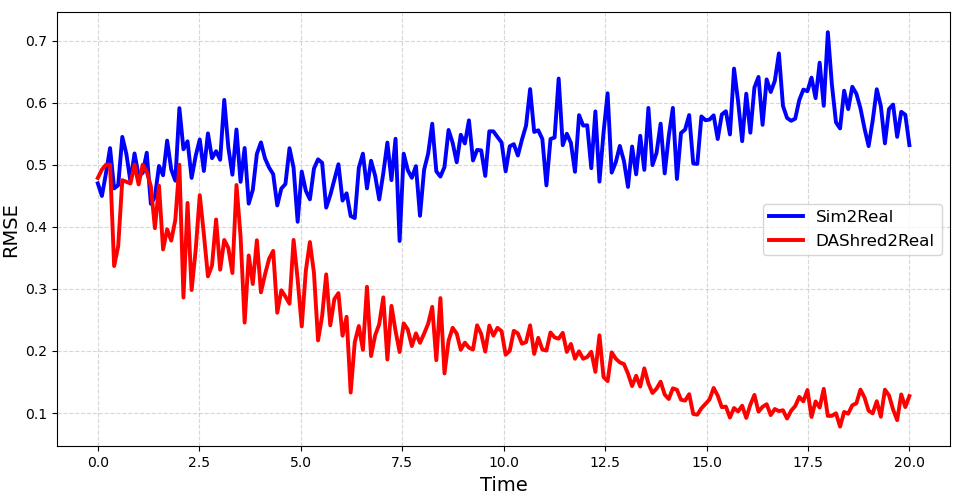}}
      \caption*{}\label{subfig:key-b}
    \end{subfigure}\hfill

    \medskip
    \begin{subfigure}{\linewidth}
      \vstretch{1.1}{\includegraphics[width=\linewidth]{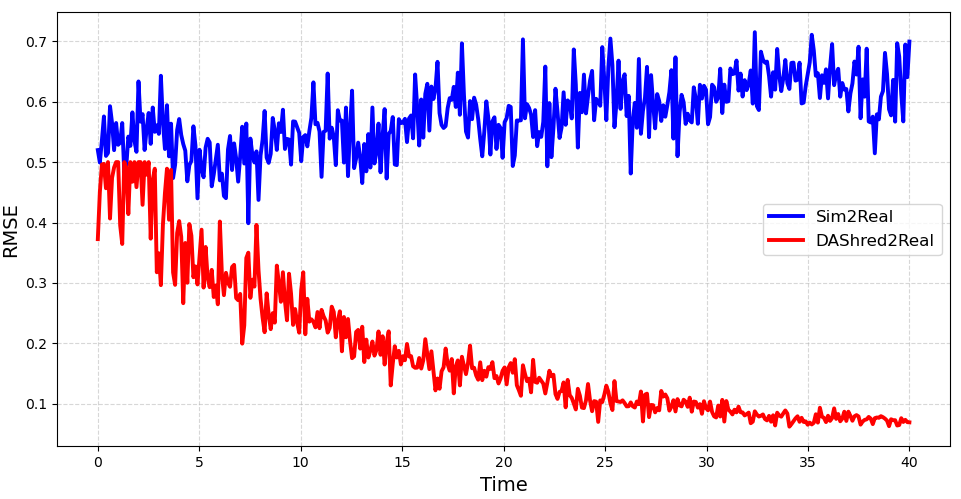}}
      \caption*{}\label{subfig:key-c}
    \end{subfigure}\hfill
  \end{minipage}
\caption{The result of 2D damped Kolmogorov flow. Figures on the same row are taken at the same timestep. The left column represents the undamped simulation model (2D Kolmogorov flow without damping); the middle column represents unknown real physics (2D Kolmogorov flow with linear damping); and the right column represents state space restored by DA-SHRED.
\label{fig:2DKol}}
\end{figure}

It is important to emphasize that achieving the strong DA-SHRED restoration performance in this example required a relatively larger number of randomly placed sensors ($p+q=20$) for the real physics to adequately span the state space. By contrast, the simulation model is more forgiving - a conventional SHRED architecture can often reconstruct the full state space using only a few sensors ($p=3$), since the simulation is self-consistent and lacks the SIM2REAL discrepancies presented in real physics.

When the number of sensors in the real physical system is low, it is still possible to match the sensor measurements reasonably well. However, this is often achieved through a reformulation of the latent space, which produces a low loss at the sensor level but can correspond to major differences in the full-field restoration. One potential source of these discrepancies is the nontrivial transient stage of the system - the SHRED may overfit to the transient behavior captured by the sensors, or conversely overfit to steady-state patterns, depending on which phase dominates the limited sensor data. Critically, because the true physical state remains unobserved at the state-space level in our problem setting, there is no straightforward way to detect these discrepancies. In other words, insufficient sensors may lead not only to inaccurate discrepancy models but also to a false sense of confidence, since standard sensor-based metrics may appear acceptable.

Increasing the number of excess sensors $q$ therefore plays a crucial role in closing the SIM2REAL gap. In practice, additional sensors provide broader coverage of state-space regions where the system exhibits relatively silent or weakly active dynamics, such as during nontrivial transient stages. By including these regions in the sensor set, the SHRED is less likely to overfit to specific patterns in the active transients, reducing latent-space misalignment and improving the robustness of full-field restoration. In this experiment, the total number of sensors could be further reduced ($p+q=10$) with additional strategies - for example, with prior knowledge about the transition from transient to steady-state behavior together with reweighting techniques that prioritize steady-state patterns. 

It is worth noting that in the 2D Kuramoto–Sivashinsky equation studied above, the transient stage was relatively silent, and thus the DA-SHRED architecture was less sensitive to sensor placement. In more general systems, however, transient dynamics can be highly nontrivial, and without prior knowledge of when the flow transitions to a steady-state, one must rely on a denser sensor network to ensure reliable restoration. In these cases, increasing the number of sensors deployed for real physics($q$) is a practical safeguard to capture the latent dynamics and close the gap for SIM2REAL discrepancies.

\subsection{2D reaction-diffusion system}

\begin{figure}[t]
    \begin{minipage}{0.55\linewidth}
  \begin{subfigure}{\linewidth}
    \begin{overpic}[width=\linewidth]{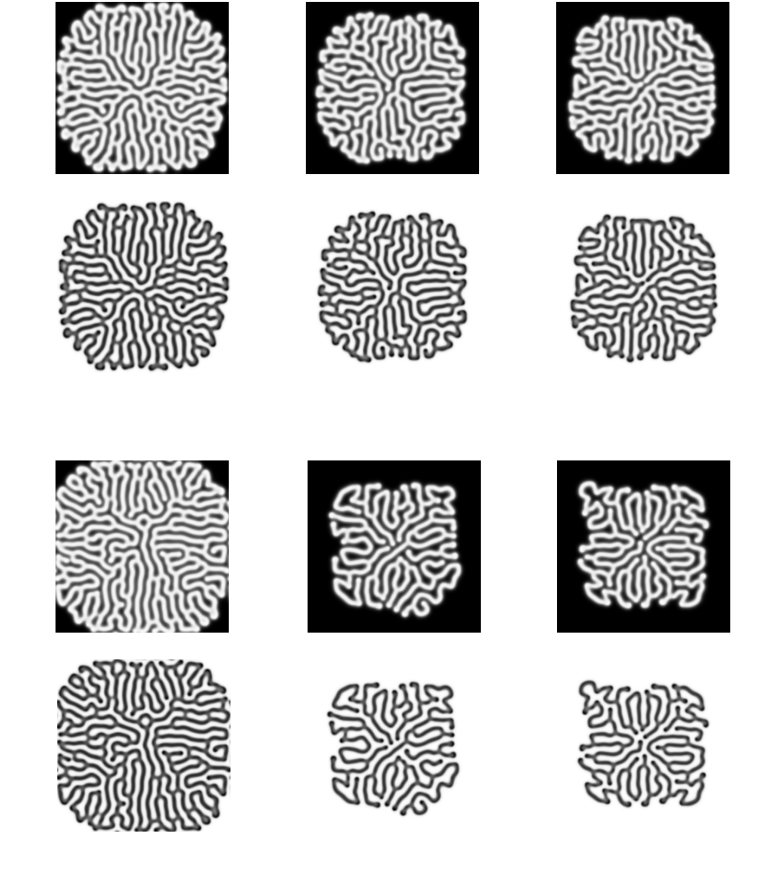}
    \put(8,101){simulation}
    \put(35,101){real physics}
    \put(62,101){DA-SHRED}
    \put(33,1){\color{black} ${{L}'} = \{(U^2V)\}$}
    \put(62,1){\color{black} ${{L}'} = \{(U^2V, {\color{gray}{V^3}})\}$}
    \put(33,53){\color{black} ${{L}'} = \{(V^2)\}$}
    \put(57,53){\color{black} ${{L}'} = \{(V^2, {\color{gray}{UV, V^3}})\}$}
    \end{overpic}
    \caption*{}\label{subfig:key-a}
  \end{subfigure}
  \end{minipage}
  \begin{minipage}{0.45\linewidth}
    \begin{subfigure}{\linewidth}
      \vstretch{1.1}{\includegraphics[width=\linewidth]{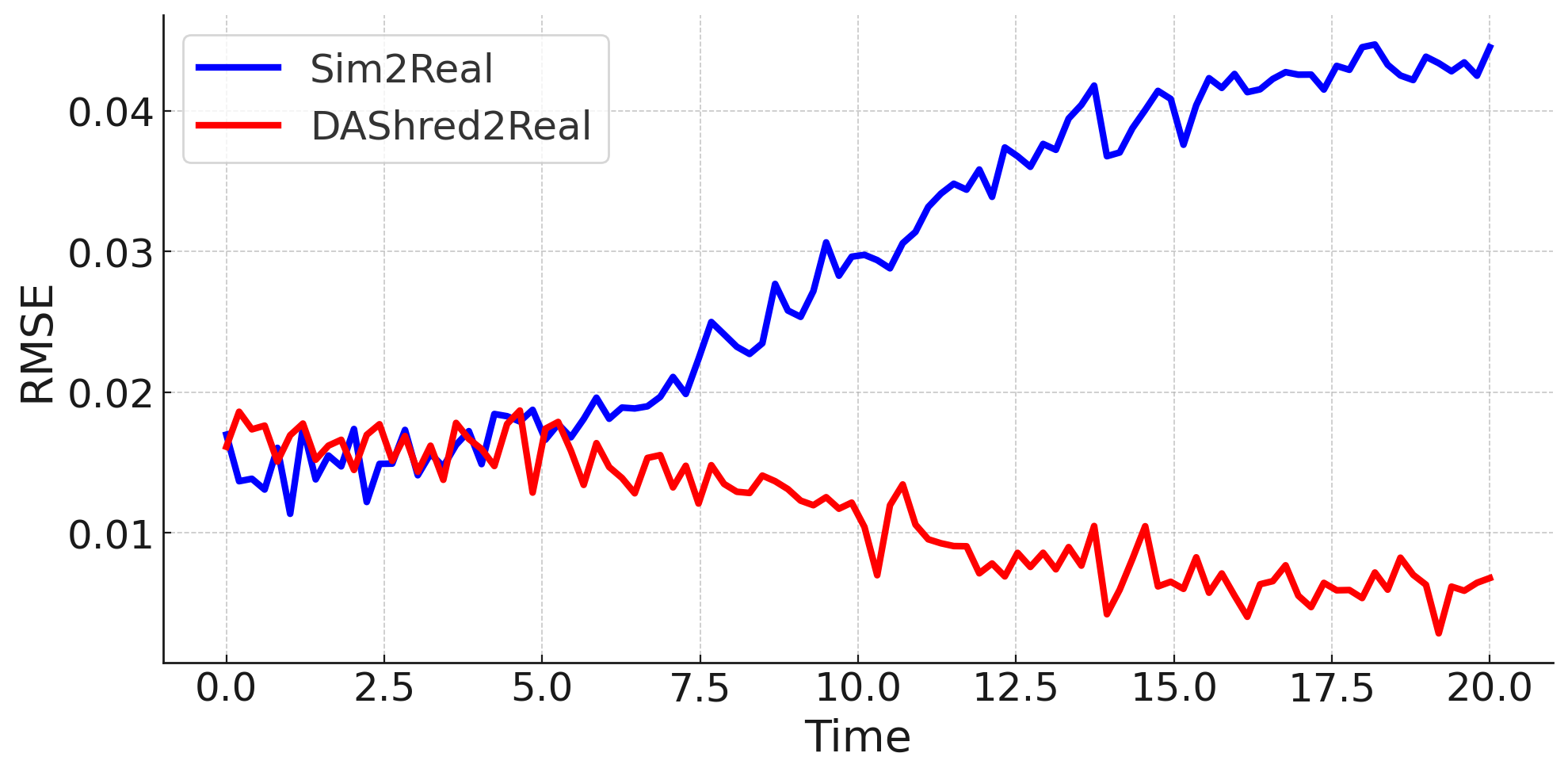}}
      \caption*{}\label{subfig:key-b}
    \end{subfigure}\hfill

    \medskip
    \begin{subfigure}{\linewidth}
      \vstretch{1.1}{\includegraphics[width=\linewidth]{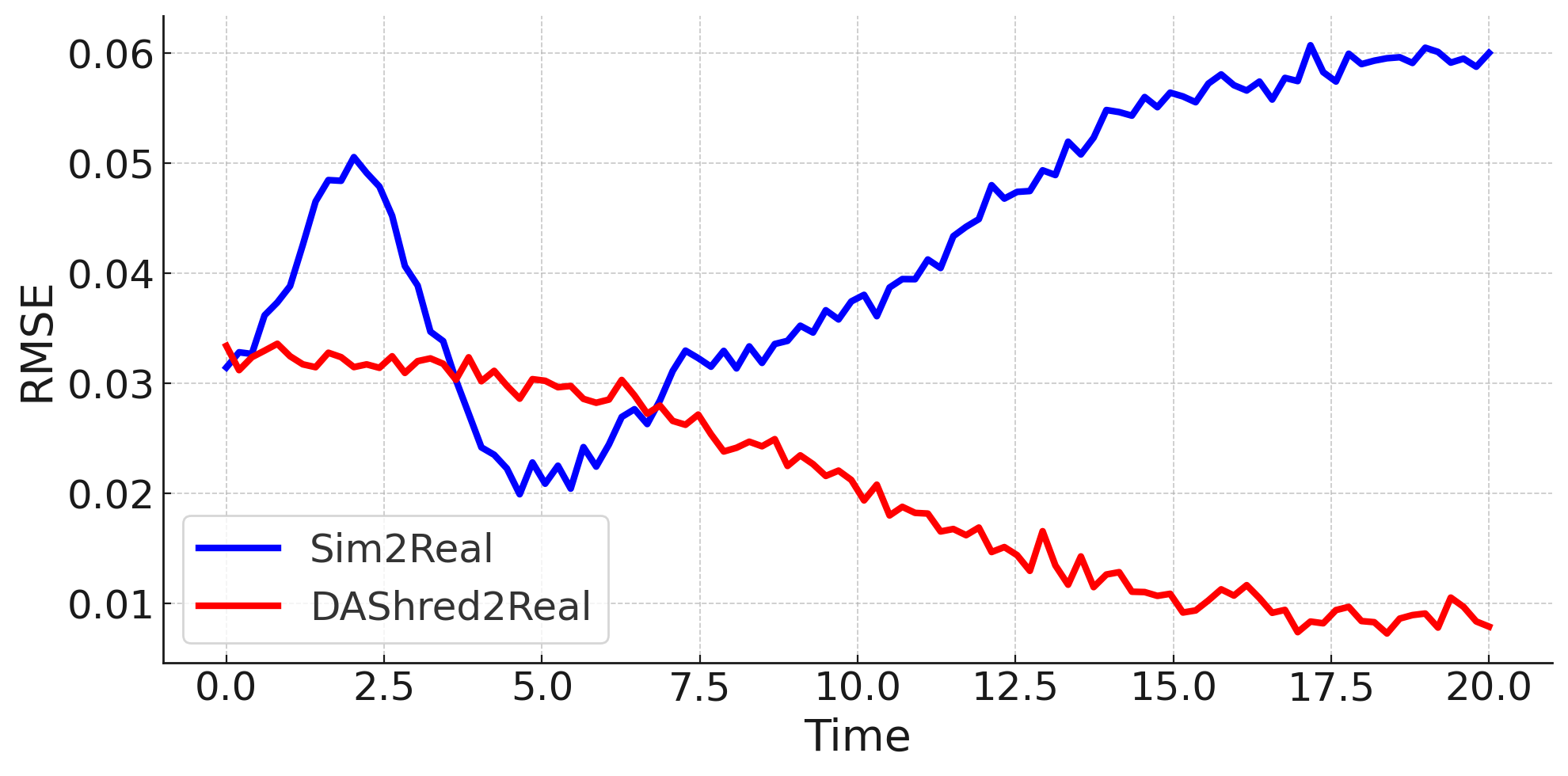}}
      \caption*{}\label{subfig:key-c}
    \end{subfigure}\hfill
  \end{minipage}
  \caption{The result of 2D damped Grey Scott model. The left figure represent the undamped simulation model; the right one represents the real model and the bottom one represents the DA-SHRED result towards convergence.}
  \label{fig:GS}
\end{figure}
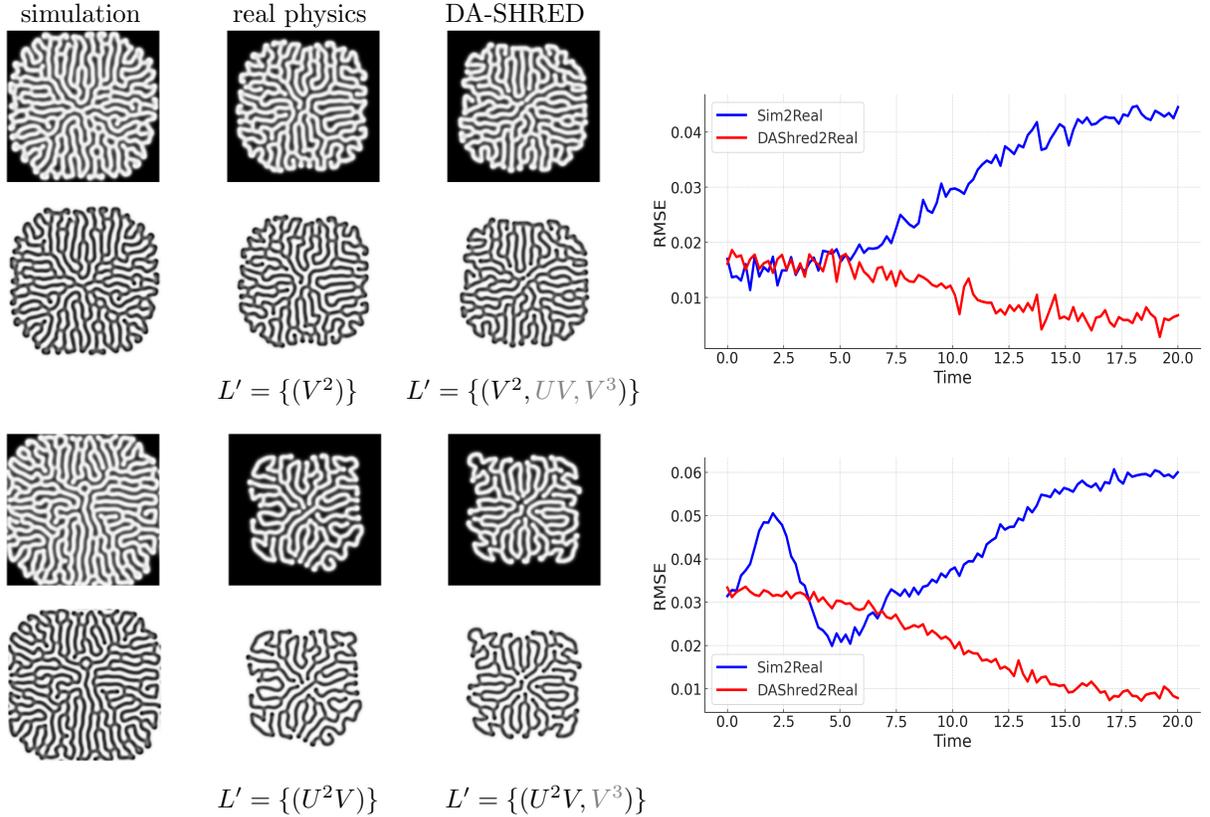

The third example we consider is the two-dimensional (2D) Gray–Scott reaction–diffusion model, which is a prototypical system for studying pattern formation and spatiotemporal dynamics in nonlinear chemical and biological systems. Originally derived as a simplified representation of autocatalytic chemical reactions~\cite{gray1984autocatalytic}, the model describes the interaction between two chemical species, $U$ and $V$, undergoing local reactions and diffusive transport. The governing equations take the form
\begin{equation}
\frac{\partial U}{\partial t} = D_u \nabla^2 U - UV^2 + F(1 - U), \qquad
\frac{\partial V}{\partial t} = D_v \nabla^2 V + UV^2 - (F + k)V
\end{equation}
where $D_u$ and $D_v$ are diffusion coefficients, while $F$ and $k$ represent feed and removal rates, respectively. Depending on parameter choices and initial conditions, the Gray--Scott model exhibits a wide variety of self-organized patterns --- including spots, stripes, labyrinths, and oscillatory structures~\cite{pearson1993complex, lee1993pattern}. It has since become a canonical system for exploring nonlinear instabilities, Turing mechanisms, and pattern formation in reaction–diffusion dynamics~\cite{meron1992pattern,roussel2004pattern, kyrychko2009control, ali2023spatiotemporal}. The simplicity and rich phenomenology of the Gray–Scott equations make them an essential benchmark for numerical experiments, data assimilation, and reduced-order modeling of spatiotemporal processes~\cite{giampaolo2022physics}

%

The true physics is similarly a damped version of it, given by
\begin{equation}
\frac{\partial U^\prime}{\partial t} = D_{u^\prime} \nabla^2 U^\prime - U^\prime {V^\prime}^2 - \alpha {V^\prime}^2 + F(1 - U^\prime), \qquad
\frac{\partial V^\prime}{\partial t} = D_{v^\prime }\nabla^2 V^\prime + U^\prime {V^\prime}^2 - (F + k)V^\prime
\end{equation}
with damping coefficient $\alpha$, and f, k are feed and kill rates, respectively. Simulation results for this are given by the top rows of Fig.~\ref{fig:GS}.  As before, the real physics for $U$ and $V$ is given by the middle row of panels, the left panels are our simulation of reality, and the right panels are the DA-SHRED updates with the error curves given in the far right.

In a second discrepancy modeling experiment, we make the damping term a functional involving both parameters as following
\begin{equation}
\frac{\partial U^\prime}{\partial t} = D_{u^\prime} \nabla^2 U^\prime - U^\prime {V^\prime}^2 - \beta {U^\prime}^2V^\prime + F(1 - U^\prime) \qquad
\frac{\partial V^\prime}{\partial t} = D_{v^\prime }\nabla^2 V^\prime + U^\prime {V^\prime}^2 - (F + k)V^\prime
\end{equation}
with damping coefficient $\beta$, and f, k are feed and kill rates, respectively. The results of this SIM2REAL gap and the ability of DA-SHRED to close the gap is given in Fig.~\ref{fig:GS} (bottom panels).
Consistent with the previous case, the middle panels show the true dynamics of $U$ and $V$, the left panels present our simulated result, and the right panels depict the DA-SHRED reconstructions with corresponding error curves in the outermost column.

\subsection{Rotating Detonation Engines (RDEs)}

Over recent decades, the concept of rotating detonation engines has emerged as a pathway towards realizing pressure-gain combustion for propulsion and power-generation systems. Unlike conventional deflagration-based combustion, which proceeds subsonically at nearly constant pressure, RDEs sustain supersonic detonation fronts that continuously propagate in the azimuthal direction of an annular combustor. It enables rapid and efficient energy release, offering potential improvements in thermal efficiency and specific impulse compared with constant-pressure Brayton-cycle engines. \cite{wolanski2013detonative}

The applications of RDEs has been most prominently investigated for aerospace applications, especially air-breathing and rocket propulsion systems. In rocket propulsion, RDEs have demonstrated the capacity to improve specific impulse and reduce propellant consumption due to their inherently higher stagnation pressures compared to conventional combustors. In air-breathing engines, they offer compact architectures with reduced turbomachinery requirements and enhanced fuel efficiency~\cite{lu2014rotating}. Moreover, RDEs are also being explored for stationary power generation, particularly as a means of improving gas-turbine cycle performance and reducing fuel consumption~\cite{sousa2017thermodynamic}.

\begin{figure}[t]
  \begin{minipage}{1.0\linewidth}
    \centering
    \begin{overpic}[width=\linewidth]{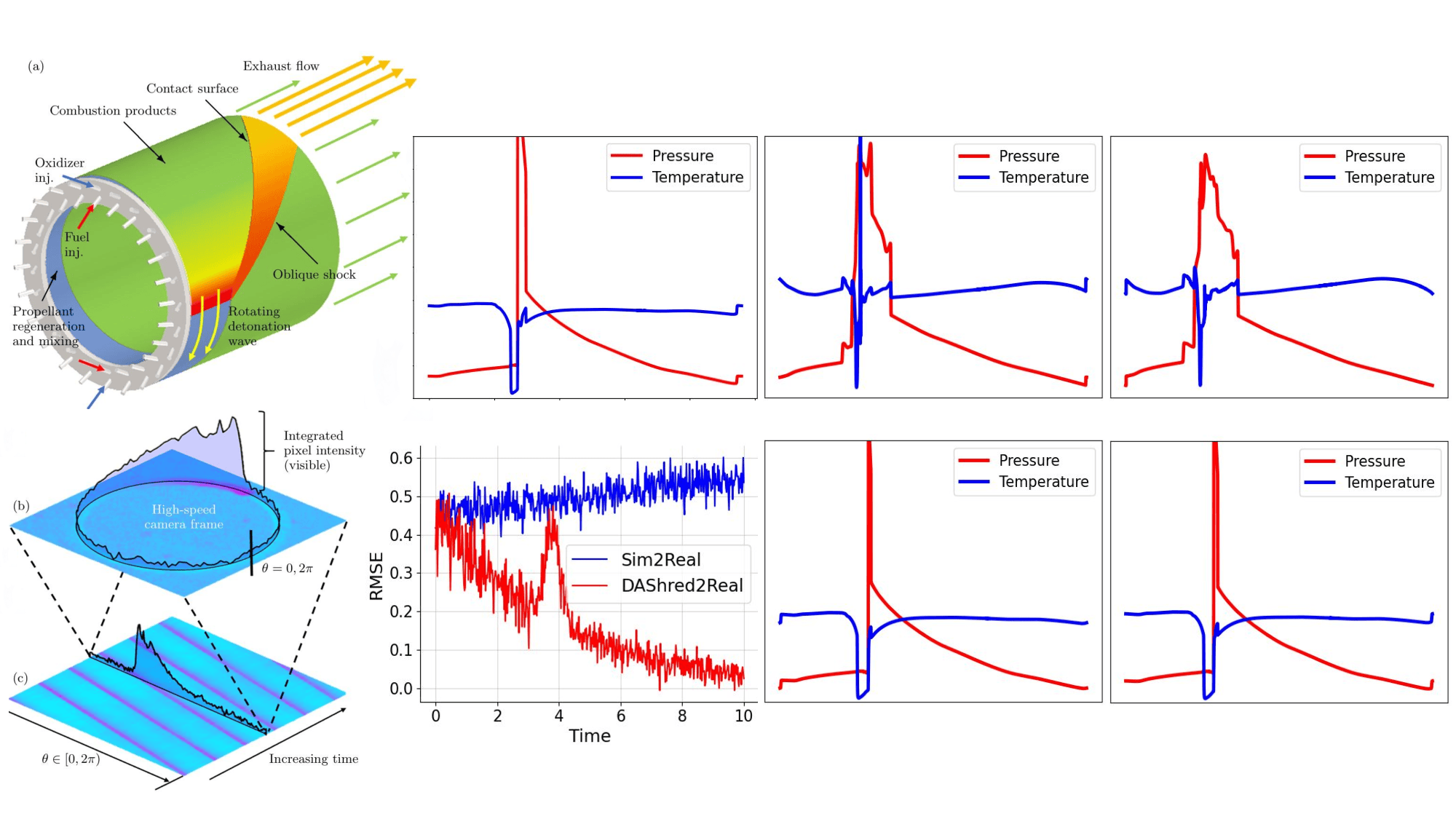}
    \put(35,48){simulation}
    \put(59,48){real physics}
    \put(82,48){DA-SHRED}
    \put(58,27){\color{black} ${{L}'} = \{({\bf u},{\bf u}^3)\}$}
    \put(79,27){\color{black} ${{L}'} = \{({\bf u},{\bf u}^3, {\color{gray}{\lambda^2 {\bf u}}})\}$}
    \put(59,6){\color{black} ${{L}'} = \{({\bf u})\}$}
    \put(82,6){\color{black} ${{L}'} = \{({\bf u})\}$}
    \end{overpic}
  \end{minipage}
  \caption{The annotated figure on the left is from Koch's et al~\cite{koch2020nonlinear} depicting the canonical flowfield of the Rotating Detonation Engine. The figures on the right depict DA-SHRED result of 1D damped RDE model. One simulation model is deployed on two different real physics settings, and the rightmost column represents the corresponding two DA-SHRED results towards convergence. The RMSE of the simulation model versus the DA-SHRED model is provided on the bottom left.
  \label{fig:rde}}
\end{figure}
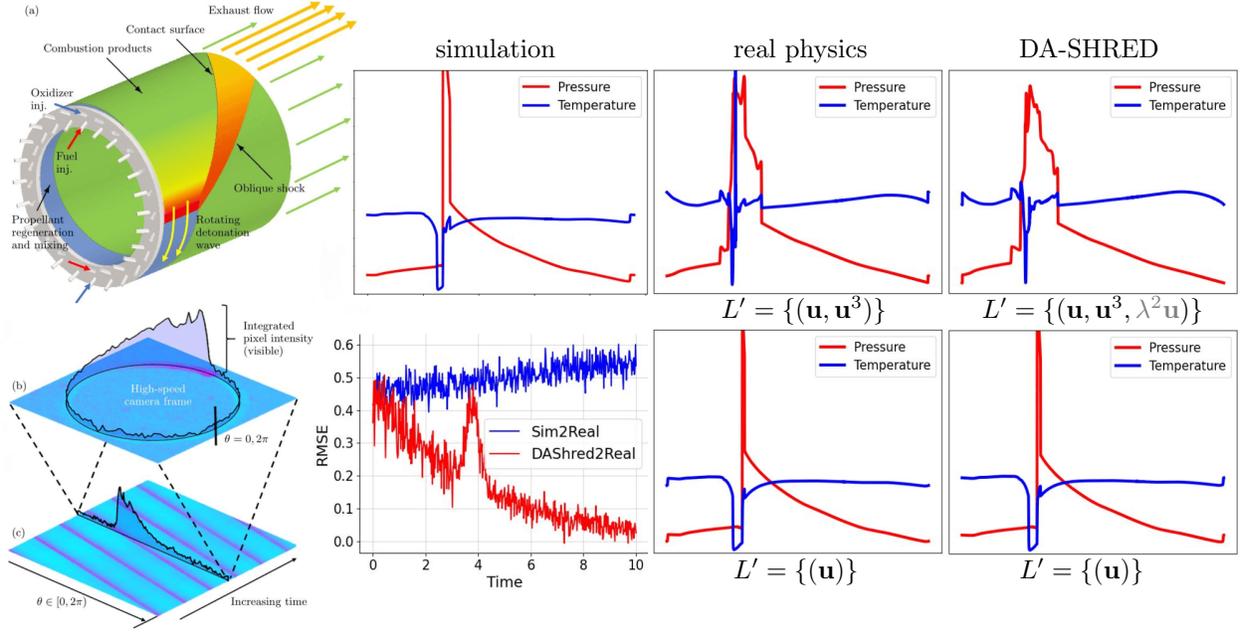

The modeling of RDEs requires capturing the coupled dynamics of compressible flow, shock propagation, and chemical reactions within an annular channel~\cite{koch2020mode,koch2020modeling,koch2021multiscale}. Such coupled processes have been analyzed in both experimental and numerical studies of RDEs, revealing the intricate interplay between detonation fronts and the reacting flow field~\cite{hishida2009fundamentals, schwer2011numerical} and showing a diverse range of bifurcations and instabilities~\cite{koch2021multiscale} that are characteristic of damp-driven systems~\cite{kutz2022universal}. These processes occur over disparate spatial and temporal scales, where detonation fronts propagate at supersonic speeds while chemical reactions and mixing evolve on much smaller scales. The interaction between the detonation front, the unsteady flow field, and the continuous injection of reactants produces strong pressure and temperature gradients that challenge conventional steady-state or quasi-one-dimensional assumptions~\cite{kawashima2017quasi}. Even simplified flow geometries exhibit complex wave structures. As discussed in recent thermodynamic analyses of RDE~\cite{koch2020modeling}, the inherently transient nature of detonation waves and their sensitivity to boundary conditions, injector dynamics, and heat release make numerical modeling both computationally demanding and physically intricate. Accurately representing these multi-physics couplings is therefore essential to predict the thermodynamic trends, stability, and efficiency of RDE operation. To reduce this complexity while retaining the essential detonation dynamics, the following simplified 1D model is considered, which captures the primary propagation behavior of the detonation front.
\begin{subequations}
\begin{align}
\frac{\partial u}{\partial t} + u \frac{\partial u}{\partial x}
= q\, k (1-\lambda) e^{\tfrac{u - u_c}{\alpha}}
- \epsilon u^2 \\
\frac{\partial \lambda}{\partial t}
= k (1-\lambda) e^{\tfrac{u - u_c}{\alpha}}
- \beta(u, u_p, s) \lambda
\end{align}
\end{subequations}
The first equation describes energy input output and the second equation describes gain depletion and recovery, where $u_p$ is the injection threshold parameter and $\beta(u, u_p, s) = \frac{s u_p}{1 + e^{r (u - u_p)}}$.

Data assimilation problems with RDEs are often difficult because the state estimations have to contend with the highly nonlinear, discontinuous, and transient nature of detonation waves. Even for the simplified simulation model, the PDEs are not only nonlinear but also extremely sensitive to initial and boundary conditions. The presence of steep gradients at the detonation front further challenges our ability to close the gap between simulation and reality as slight errors can easily propogate and lead to different modes and instabilities. Therefore, data assimilation must not only correct the evolving state but also converge on the inferrence of the damping parameters. Parameter estimation in such nonlinear, stiff systems are notoriously unstable.

In this example, we use the above system of equations as the simulation model, while the real physics is a damped version. We explored three different damping scenarios, given by adding a diffusive effect on overall energy, perturbing the gain depletion, and perturbing the energy output.  The true physics model we construct is the following governing set of equations
\begin{subequations}
\begin{align}
\frac{\partial u^\prime}{\partial t} + u^\prime \frac{\partial u^\prime}{\partial x}
= q\, k (1-\lambda^\prime) e^{\tfrac{u^\prime - u^\prime_c}{\alpha}}
- \epsilon {u^\prime}^2 + \epsilon_2(u^\prime - {u^\prime}^3) \\
\frac{\partial \lambda^\prime}{\partial t}
= k (1-\lambda^\prime) e^{\tfrac{u^\prime - u^\prime_c}{\alpha}}
- \beta(u^\prime, u^\prime_p, s) \lambda^\prime .
\end{align}
\end{subequations}
Simulations of reality and the computational model are given in Fig.~\ref{fig:rde} for two different discrepancy models.

Figure~\ref{fig:rde} shows how the DA-SHRED model can successfully track the detonation front without delays while keeping the conservation structure of the governing equations. When distortions are not too strong and abrupt, perturbations of different kinds can be reconciled well with the data assimilation model. 
As before, the real physics for $U$ (red line) and $\lambda$ (blue line) are given by the middle row of panels, the left panels are our simulation of reality, and the right panels are the DA-SHRED updates with the error curves given in the bottom left.   As illustrated in every model considered, the DA-SHRED quickly lowers the SIM2REAL error by an order of magnitude as data is collected.

\section{Discrepancy Modeling with DA-SHRED}

\begin{figure}[t]
    \begin{minipage}{0.90\linewidth}
  \begin{subfigure}{1.1\linewidth}
    \begin{overpic}[width=\linewidth]{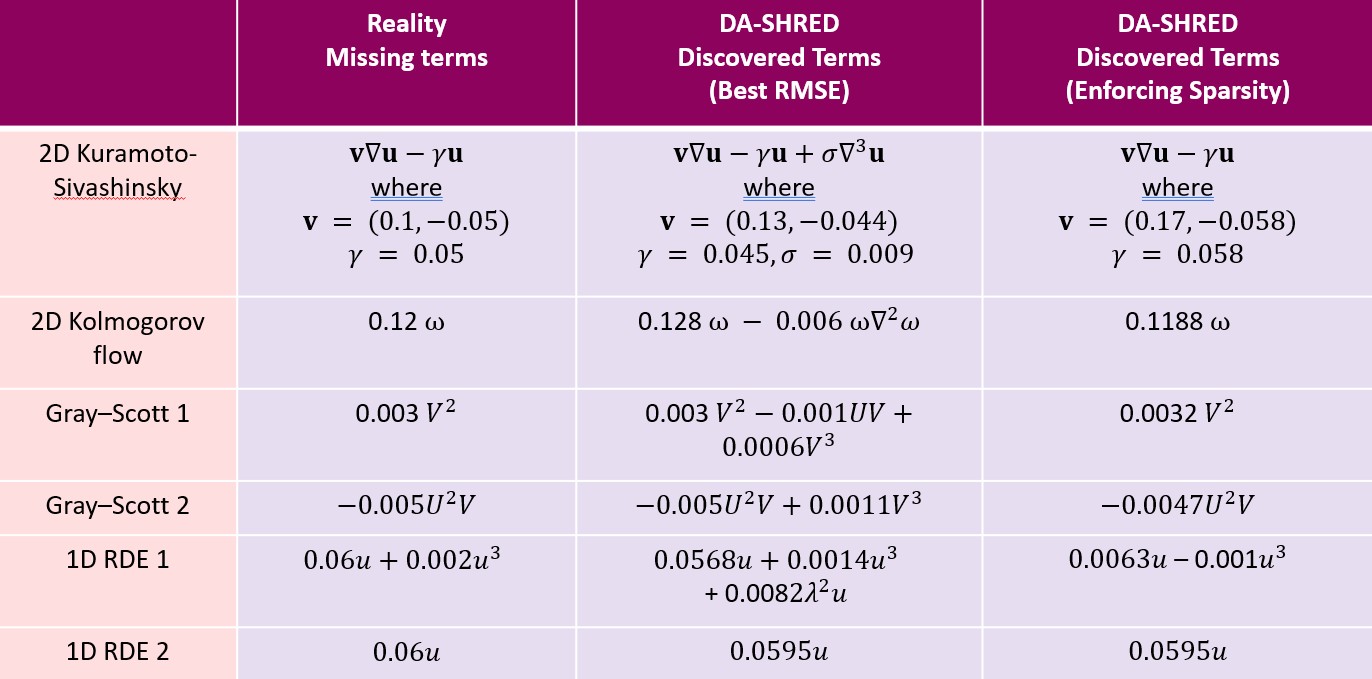}
    \end{overpic}
    \caption*{}\label{subfig:key-a}
  \end{subfigure}
  \end{minipage}
  
  \caption{The full list of missing terms $L'$ in real physics compared to our DA-SHRED restoration $L^\prime$ for the experiments outlined in the previous section.  The two columns illustrating the recovery highlight the trade-off between targeting the reduction of RMSE (middle column) and sparsity (right column).  With an unknown $L'$, a balance between RMSE and sparsity must be hyper-parameter tuned for the best discrepancy model.}
  \label{fig:discrepancy}
\end{figure}

In the DA framework introduced above, the SIM2REAL gap is modeled as the discrepancy between the simulation model $\bf{u_t} = N(\bf{u},\bf{x},t)$, which captures the approximate governing physics, and the real, unknown dynamics $\bf{u^\prime_t} = M(\bf{u^\prime},\bf{x},t) = N(\bf{u^\prime},\bf{x},t) + L' (\bf{u^\prime},\bf{x},t)$, where $L'$ denotes a small, state-space dependent perturbation term that encodes the unknown or unresolved physics. While the data assimilation step reduces the effective mismatch between the two systems, it does not in itself provide the explicit functional structure of the missing physics. Recovering the exact form of the missing functional terms $L'$ remains an open inverse problem. To address this challenge, we present below two SINDy-based~\cite{brunton2016discovering} algorithms that augment the DA-SHRED framework and aim to identify the unknown terms directly from SIM2REAL trajectories by sparse regression (\ref{eq:Lprime}).

In practical settings, the recovery of $L'$ has an inherent trade-off between predictive accuracy and structural sparsity. As shown in the original SINDy formulation~\cite{brunton2016discovering}, models that achieve extremely low prediction error often do so by selecting additional terms in the library to lower RMSE, effectively over-parameterizing the dynamics. Similarly, the DA-SHRED model that minimizes RMSE tends to incorporate additional candidate functions - reflecting noise, numerical artifacts, or overfitting to uninformative directions - resulting in expressions that may deviate from the true physical process. However, identifying the stopping point for sparsification remains challenging, as excessive sparsity risks eliminating essential operators, whereas insufficient sparsity retains spurious terms. This characteristic tradeoff is illustrated in in Figure 7, where sparser DA–SHRED reconstructions more accurately recover the true governing operators despite yielding slightly worse numerical fits to the data in regards to RSME.

The ability to determine how many terms are missing from the simulation model $N(\bf{u^\prime},\bf{x},t)$  depends critically on available physical prior knowledge. In systems derived from well-studied PDEs, symmetry arguments, conservation laws, or known operator hierarchies,  the missing physics can be constrained within a small set of candidate terms, such as higher-order diffusion, nonlinear advection, or corrections in the form of divergence. In these situations, classical closure models and scale-separation theory~\cite{pope2001turbulent, sagaut2006large} often provide clear expectations about the number and type of missing contributions, making sparse discovery more reliable.

Conversely, in systems where external forcing, boundary interactions, unmeasured actuation, or nonlocal effects play a dominant role, little or no prior knowledge exists regarding the dimensionality of the missing operator space. In such cases, the accuracy–sparsity trade-off becomes even more influential, and overfitting or underfitting risks are elevated due to the lack of constraints on the candidate library.

\subsection{In the compressed space}

\begin{algorithm}[t]
\caption{Compressed search SINDy}\label{alg:cap}
\begin{flushleft}
 \textbf{Input:} 

A simulation model $\dot{\textbf u} = N(\textbf u,\textbf x,t)$
 
A multivariate time-series of simulation state space $\{\textbf u_i\}_{i=1}^{m}$ and corresponding sensor measurements of the system $\{\textbf s_i\}_{i=1}^{m}$

A multivariate time-series of real state space $\{\textbf u_i^\prime\}_{i=1}^{m}$ and corresponding sensor measurements of the system $\{\textbf s_i^\prime\}_{i=1}^{m}$
 
A DA-SHRED architecture $\mathcal{G}(\{\textbf s_i^\prime\}_{i=t-k}^{t}; \textbf W_R) = \mathcal{T}(\mathcal{G}(\{\textbf s_i\}_{i=t-k}^{t}; \textbf W_S); \textbf W_T)$ and a shallow decoder $\mathcal{F}$ for restoration from the compressed latent space

A library of functional candidates $\boldsymbol\Theta = \{\theta_j\}_{j=1}^n$ \\

\textbf{Algorithm:} 

At time step t, perform the following

Generate $\tilde u_{t}^\prime = \mathcal{F} (\mathcal{G}(\{\textbf s_i^\prime\}_{i=t-k}^{t}; \textbf W_R))$ of the current restored real state space 

Compute $\tilde u_{t,j}^\prime = \tilde u_{t-1,j}^\prime + (N(\tilde u_{t-1,j}^\prime,\textbf x,t) + \theta_j(\tilde u_{t-1,j}^\prime,\textbf x,t))dt$ of perturbed state space

Acquire proxies in the latent space using the trained SHRED model $G_t(\dot{\textbf u}) = \mathcal{G}(\tilde u_{t}^\prime; \textbf W_S)$ and $H_t(\boldsymbol\Theta) = \mathcal{G}(\{\tilde u_{t,j}^\prime\}_{j=1}^n; \textbf W_S)$

Optimize $\xi \in argmin_{\xi} \frac{1}{2} ||G_t(\dot{\textbf u}) -H_t(\boldsymbol\Theta) \xi||_2^2 + K(\xi)$, where $K(\xi) = ||\xi||_1$ is a regularizer chosen to promote sparsity in $\xi$. It can be solved by sparsity-promoting SINDy method

\end{flushleft}
\end{algorithm}

The first algorithm we propose is a SINDy-based regression procedure applied directly to the difference dynamics expressed in the compressed (latent) space. In a well-behaved DA–SHRED assimilation, we obtain accurate latent-space representations of both the simulation operator $N(\bf{u^\prime},\bf{x},t)$ and the real dynamics $M(\bf{u^\prime},\bf{x},t) = N(\bf{u^\prime},\bf{x},t) + L'$. Consequently, the discrepancy between these two representations - in the form of a latent-space difference matrix -encodes the action of the missing physics  projected onto the compressed manifold. Recovering the physical operators responsible for this discrepancy reduces the algorithm to identifying the functional terms whose latent-space signatures reconstruct this difference. This motivates the use of sparsity-promoting regression to extract only the dominant contributors.

To accomplish this, we first evaluate the compressed reassembly of every candidate functional in the original library by pushing each operator through the encoder and computing its corresponding action in latent space. This yields a transformed library whose columns represent how much each candidate operator influences the compressed coordinates. With this latent-space library in hand, we then apply a sparsity-promoting SINDy regression to determine the minimal combination of candidate operators whose compressed signatures reconstruct the observed difference dynamics. The resulting sparse coefficient vector identifies the specific physical functionals that together constitute $L'$, thereby completing the discrepancy modeling.

\subsection{Advancing in the compressed space}

For certain state spaces, the limitations of the above algorithm become apparent. First, computing the compressed response of every perturbation functional defined on the original state space can be computationally demanding, especially when the functional library is large or when high-dimensional fields must be encoded repeatedly. Second, because the algorithm relies on derivatives obtained via finite-difference or interpolation schemes, the resulting numerical artifacts may be amplified during compression. This noise propagation can corrupt the latent-space representations and ultimately prevent the sparse regression from converging to a stable solution. These issues highlight the need for alternative formulations that avoid noise accumulation and reduce the computational burden associated with evaluating the full scale operator library.

To address the limitations above, we propose a second SINDy-based algorithm that avoids operating on the difference matrix formed from compressed reassemblies. Instead, this method constructs the target regression signal directly from timestep-to-timestep differences in the latent space. By focusing on the temporal increments of the trajectory, the algorithm bypasses the need to repeatedly evaluate and compress the full operator library, thereby reducing both computational cost and sensitivity to derivative estimates and noise.

Moreover, rather than applying each candidate functional to the entire original state space, we restrict the perturbation evaluations to local neighborhoods surrounding the sensor locations. Although these localized perturbations are not sufficient to reproduce the exact latent-space temporal differences, they provide informative proxies whose aggregated behavior reflects the underlying physical corrections present in the true dynamics. In practice, these localized operator responses capture the dominant directional influences that DA–SHRED extracts with its time-delay embedding, and thus serve as meaningful surrogate signatures for sparse regression. By leveraging these localized proxies instead of full-field operator projections, the method substantially reduces computational complexity while still enabling the recovery of the functionals comprising $L'$ in the missing physics.

\begin{algorithm}[t]
\caption{Compressed advancing SINDy}\label{alg:cap}
\begin{flushleft}
 \textbf{Input:}

 A simulation model $\dot{\textbf u} = N(\textbf u,\textbf x,t)$
 
A multivariate time-series of simulation state space $\{\textbf u_i\}_{i=1}^{m}$ and corresponding sensor measurements of the system $\{\textbf s_i\}_{i=1}^{m}$

A multivariate time-series of real state space $\{\textbf u_i^\prime\}_{i=1}^{m}$ and corresponding sensor measurements of the system $\{\textbf s_i^\prime\}_{i=1}^{m}$
 
A DA-SHRED architecture $\mathcal{G}(\{\textbf s_i^\prime\}_{i=t-k}^{t}; \textbf W_R) = \mathcal{T}(\mathcal{G}(\{\textbf s_i\}_{i=t-k}^{t}; \textbf W_S); \textbf W_T)$ and a shallow decoder $\mathcal{F}$ for restoration from the compressed latent space

A library of functional candidates $\boldsymbol\Theta = \{\theta_j\}_{j=1}^n$ \\

\textbf{Algorithm:} 

At time step t, perform the following - 

Generate $\tilde u_{t}^\prime = \mathcal{F} (\mathcal{G}(\{\textbf s_i^\prime\}_{i=t-k}^{t}; \textbf W_R))$ of the current restored real state space 

Compute $\tilde u_{t,j}^\prime = \tilde u_{t,j}^\prime + \sum _{i=1}^m \theta_j(\tilde u_{t,j}^\prime,\textbf s_i,t)dt$ of possible local perturbations at sensor locations

Acquire proxies in the latent space using the trained SHRED model $G_t(\dot{\textbf u}) = (\mathcal{G}(\{\textbf s_i^\prime\}_{i=t+1-k}^{t+1}; \textbf W_R) + \mathcal{G}(\{\textbf s_i^\prime\}_{i=t-k}^{t}; \textbf W_R)$ and $H_t(\boldsymbol\Theta) = \mathcal{G}(\{\tilde u_{t,j}^\prime\}_{j=1}^n;  \textbf W_S) - \mathcal{G}(\tilde u_{t}^\prime; \textbf W_S)$

Optimize $\xi \in argmin_{\xi} \frac{1}{2} ||G_t(\dot{\textbf u}) -H_t(\boldsymbol\Theta) \xi||_2^2 + K(\xi)$, where $K(\xi) = ||\xi||_1$ is a regularizer chosen to promote sparsity in $\xi$. It can be solved by sparsity-promoting SINDy method

\end{flushleft}
\end{algorithm}

The results of the recovery algorithms are already presented in Figs.~\ref{fig:2DKS}--\ref{fig:rde}.  Specifically, each of these figures shows the true $L'$ versus the discovered terms of $L'$ from DA-SHRED.  Figure~\ref{fig:discrepancy} also shows the trade-offs of the DA-SHRED algorithm between RMSE and sparsity.  Focusing on RMSE can lead to additional terms being found in the sparse regression framework.  Alternatively, the focus on sparsity tends to recover the correct missing physics terms while producing a higher RMSE.

\section{Conclusion and Discussions}

In this work, we introduced the Data Assimilation with Shallow Recurrent Decoder (DA-SHRED) framework, a hybrid machine learning and physics-informed architecture designed to bridge the simulation-to-reality (SIM2REAL) gap in dynamic systems. By leveraging a reduced-order latent space learned from simulation data and deploying it with sparse sensor measurements on the real physical system, DA-SHRED achieves accurate state reconstruction and identifies missing physics in complex, nonlinear systems. The framework extends the foundational SHRED model by incorporating data assimilation principles and embedding a sparse identification of nonlinear dynamics (SINDy) regression within the latent space to extract parsimonious model representation of the discrepancy between modeled and observed dynamics.

Through a series of challenging test cases - including the 2D damped Kuramoto–Sivashinsky equation, 2D Kolmogorov flow, Gray-Scott reaction–diffusion system, and rotating detonation engine (RDEs) - we demonstrated that DA-SHRED model can robustly reconstruct unobserved system states and capture perturbative and damping effects that are absent from the original simulation models. The framework exhibits rapid convergence once sufficient temporal information is gathered, even under sparse sensing constraints. Notably, in the RDE example, the DA-SHRED model captured broadening behavior of the detonation front induced by injector dynamics, emphasizing the framework’s ability to reconcile simplified computational models with real-world complexity.

Theoretically, we established that DA-SHRED preserves the representational structure of the underlying system under suitable assumptions. For linear systems, the closure of the SIM2REAL gap can be shown analytically, while for nonlinear systems, convergence needs to be validated computationally. Moreover, in the supplemental material, we provide further connection to port-Hamiltonian systems which gives a rigorous example for understanding eigenspace preservation under perturbations, supporting the assumption that the latent basis trained on simulation data remains valid when applied to perturbed real physics.

Several limitations and potential avenues for future work naturally emerge from this work. The current DA–SHRED formulation relies on the assumption that the simulation and real systems share a sufficiently similar latent basis. When significant modal shifts occur - due to nonstationary dynamics, regime changes, or structural mismatches - the assimilation accuracy may degrade, which in turn complicates the reliable discovery of missing functionals. To address this, incorporating adaptive basis updates or transfer-learning mechanisms can potentially allow the latent representation to adjust dynamically to new regimes, while preserving consistency with previously learned structures.

An adaptive SHRED framework can extend beyond merely assimilating data into a fixed latent space, enabling the model to navigate changing conditions autonomously. For instance, in settings involving nonstationary dynamics, changing boundaries, or heterogeneous data quality, new latent embeddings could be constrained to remain close to historical structures unless the data provides strong evidence of a regime shift. Mechanistically, this could be achieved using multi-timescale updates, such as a fast–slow clock architecture with coupled LSTMs or other recurrent networks, allowing rapid adjustments to transient deviations while maintaining stability over longer-term trends. Such an approach could enhance robustness to evolving dynamics, improve the fidelity of missing physics discovery, and expand the applicability of DA–SHRED to more complex, real-world systems.

Additionally, multi-scale systems which exhibit dynamics on widely separated temporal or spatial scales are particularly challenging to model in the current framework. Fast, high-frequency fluctuations may be underrepresented in latent embeddings that primarily capture dominant, slow modes. As a result, interactions between slow and fast processes - and the missing physics associated with them - can be difficult to accurately identify and assimilate. Addressing these challenges requires methods capable of resolving multiple scales simultaneously, such as scale-aware functional libraries or multi-timescale latent updates. Extensions of SINDy designed for  multiscale and noisy data settings~\cite{champion2019discovery, fasel2022ensemble} provide a valuable starting point, but their principles need to be leveraged within the SHRED architecture to address similar problems. Stochastic systems also pose significant challenges because the underlying deterministic dynamics can be obscured. In the current framework, latent embeddings and discrepancy modeling may inadvertently capture noise as spurious dynamics and fail to resolve subtle stochastic effects. Prior work on SINDy-type regressions~\cite{boninsegna2018sparse} provides guidance by demonstrating approaches for identifying both drift and diffusion terms in stochastic differential equations, but assimilation efforts on stochastic systems remain fundamentally challenging and may require libraries of candidate functions that explicitly include stochastic terms.

Finally, as recent developments in physics-informed and PDE-aware deep learning architectures have similarly aimed to merge data-driven inference with governing physical constraints~\cite{raissi2019physics, geneva2020modeling}, data assimilation for PDEs with nonlocal constraints still presents significant challenges. Unlike local PDE systems, where updates depend primarily on nearby states, nonlocal terms - such as integral or fractional operators - introduce global dependencies that complicate the SHRED architecture. This global coupling makes traditional assimilation techniques less effective, as small perturbations can propagate and subject to constrains nonlocally, leading to instability or slow convergence. Addressing these issues may require specialized reformulations, such as hierarchical or multi-resolution representations to separate local and global effects, or latent space embeddings that explicitly capture nonlocal interactions. Such strategies might enable DA-SHRED framework to handle a broader class of PDEs with complex, nonlocal dynamics.

Overall, DA-SHRED demonstrates a powerful synergy between data assimilation, reduced-order modeling, and physics-informed model discovery. By unifying temporal encoding, sparse sensing, and interpretable discrepancy modeling within a single architecture, it offers a promising pathway toward real-time, data-efficient closure of the SIM2REAL gap in scientific and engineering applications.

\section*{Code and Data}  All code and data can be found at:  https://github.com/Capella22/DA-SHRED

\section*{Acknowledgements}

This work was supported in part by the US National Science Foundation (NSF) AI Institute for Dynamical Systems (dynamicsai.org), grant 2112085. JNK further acknowledges support from the Air Force Office of Scientific Research  (FA9550-24-1-0141) and from the AFOSR/AFRL Center of Excellence in Assimilation of Flow Features in Compressible Reacting Flows under award number FA9550-25-1-0011, monitored by Dr. Chiping Li and Dr. Ramakanth Munipalli.

\printbibliography{}

\newpage
\appendix

\section*{\centering Supplemental Material}

\section{Preservation of the eigenspace}

In the formulation of our model, we made a bold assumption that provided we have enough sensors, the real physics in the latent space could be well-represented by the same eigenspace we have for the simulation model. Such behavior cannot be easily shown, as introduced in Section 2.1, and may not even hold for general dynamical systems. However, for a certain type of dynamical system and damping models that are port-Hamiltonian, eigenspaces can be deliberately designed to be preserved, which allow us to use the same basis for simulation and perturbed reality\cite{van2006port}.

\subsection{A simple case - damped pendulum}

In this example we use a damped pendulum as our real physics to be studied, which is given by:

\[
\begin{aligned}
    \dot{\omega} &= -\frac{g}{l} \sin(\theta) - \frac{d}{ml^2} \omega + \frac{u}{ml^2} 
    \\
    \dot{\theta} &= \omega
\end{aligned}
\]

where $\theta$ is the pendulum angle, $\omega$ is the angular momentum, d is the damping coefficient, and u is external input.

Therefore, the Hamiltonian function, representing the total energy, is given by
\[
H(\theta, \omega) = \frac{1}{2} ml^2 \omega^2 + mgl(1 - \cos \theta).
\]

And in the port-Hamiltonian architecture, the system can be reformulated as:

\[
\begin{bmatrix} \dot{\theta} \\ \dot{\omega} \end{bmatrix} =
\left( \begin{bmatrix} 0 & 1 \\ -1 & 0 \end{bmatrix}
- \begin{bmatrix} 0 & 0 \\ 0 & \frac{d}{ml^2} \end{bmatrix} \right)
\begin{bmatrix} \frac{\partial H}{\partial \theta} \\ \frac{\partial H}{\partial \omega} \end{bmatrix}
+ \begin{bmatrix} 0 \\ 1 \end{bmatrix} u.
\]

where
\begin{itemize}
    \item $J = \begin{bmatrix} 0 & 1 \\ -1 & 0 \end{bmatrix}$ is skew-symmetric.
    \item $R = \begin{bmatrix} 0 & 0 \\ 0 & \frac{d}{ml^2} \end{bmatrix}$ imports dissipation through damping.
\end{itemize}

Therefore, with the simulation model corresponds to a simple pendulum while the underlying physics reflect a damped pendulum, the eigenspace structure of the state space remains preserved. Consequently, one can expect the deployment on this eigenspace to satisfy the same favorable criteria, yielding accurate matches even with a limited number of sensors, as illustrated in Figure~\ref{fig:pen}. The DA-SHRED model effectively bridges the discrepancy, despite the presence of a moderate $5\%$ noise level in the input data.

\begin{figure}[H]
    \begin{minipage}{0.95\linewidth}
  \begin{subfigure}{1.1\linewidth}
    \begin{overpic}[width=\linewidth]{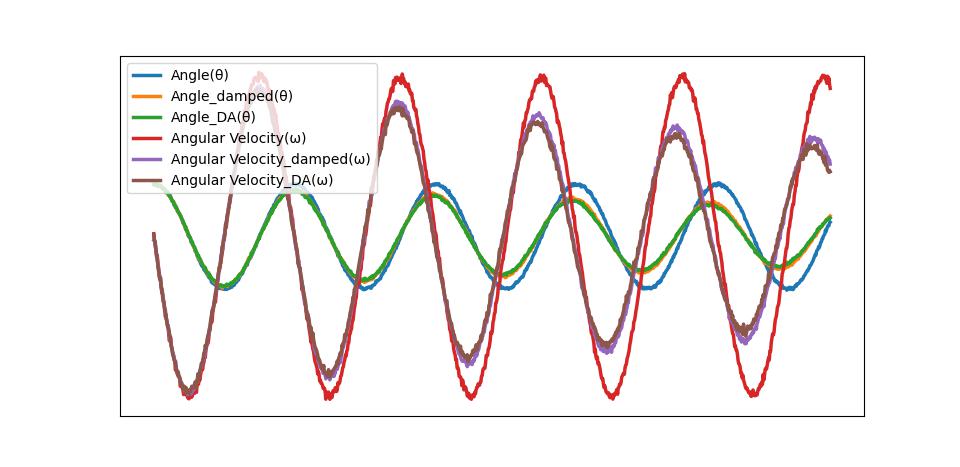}
    \end{overpic}
    \caption*{}\label{subfig:key-a}
  \end{subfigure}
  \end{minipage}
  
  \caption{The DA-SHRED restoration of damped pendulum dynamics. The angular displacement and angular velocity are presented for three cases: the simulated simple-pendulum model, the unknown true dynamics (damped pendulum), and the DA-SHRED reconstruction.\label{fig:pen}}
\end{figure}

\subsection{Detailed setup and notation}

Let $\mathcal{X}\subseteq\mathbb{R}^n$ be open and simply connected. A finite-dimensional port-Hamiltonian system (PHS) on $\mathcal{X}$ with inputs $u\in\mathbb{R}^m$ is written in coordinates as
\begin{equation}\label{eq:phs_basic}
\dot{x} = \big(J(x) - R(x)\big)\nabla H(x) + G(x)\,u,\qquad
y = G(x)^\top \nabla H(x),
\end{equation}
where:
\begin{itemize}
  \item $H\in C^2(\mathcal{X})$ is the Hamiltonian (energy) function,
  \item $J(x)\in\mathbb{R}^{n\times n}$ is skew-symmetric: $J(x)^\top = -J(x)$ (structure/interconnection),
  \item $R(x)\in\mathbb{R}^{n\times n}$ is symmetric positive semidefinite: $R(x)^\top = R(x) \succeq 0$ (dissipation),
  \item $G(x)\in\mathbb{R}^{n\times m}$ is the input (port) map.
\end{itemize}
We write $\mathrm{d}H(x)$ or $\nabla H(x)$ interchangeably for the exterior derivative / gradient.

In the next few sections, we show that under natural perturbations the perturbed dynamics can again be written in the form \eqref{eq:phs_basic} (possibly with modified $\widetilde H,\widetilde J,\widetilde R,\widetilde G$), i.e. the PHS ``base'' (the algebraic/skew/symmetric structure and the Dirac-type interconnection) persists.

\subsection{Differential-forms perspective and the integrability condition}

We regard $\nabla H$ as the exact 1-form $\mathrm{d}H$. Suppose a perturbation of the gradient-term appears as an additive 1-form $\delta(x)\in\Omega^1(\mathcal{X})$ so that the perturbed dynamics read
\begin{equation}\label{eq:pert_grad}
\dot{x} = (J-R)\big(\nabla H + \delta\big) + G\,u.
\end{equation}

\begin{prop}[Exactness / Poincaré lemma]\label{prop:poincare}
If $\mathcal{X}$ is simply connected and $\delta$ is a $C^1$ 1-form on $\mathcal{X}$ with exterior derivative $d\delta=0$ (i.e. $\delta$ is closed), then there exists a scalar function $\varepsilon\in C^2(\mathcal{X})$ such that $\delta = \mathrm{d}\varepsilon$. Consequently, defining $\widetilde H := H + \varepsilon$ yields
\[
\nabla \widetilde H = \nabla H + \delta,
\]
and the perturbed system \eqref{eq:pert_grad} is of port-Hamiltonian form with the same $J,R,G$ and Hamiltonian $\widetilde H$.
\end{prop}

\begin{proof}
This is a direct application of the Poincaré lemma: on a simply connected domain every closed 1-form is exact. If $d\delta = 0$ then there exists $\varepsilon$ with $\mathrm{d}\varepsilon=\delta$. Set $\widetilde H=H+\varepsilon$. Substitute into \eqref{eq:pert_grad} to obtain $\dot x=(J-R)\nabla\widetilde H+G u$, which has the PHS form with unchanged $J,R,G$.
\end{proof}

\begin{remark}
The condition $d\delta=0$ is necessary for a global scalar potential $\varepsilon$ to exist. If $\delta$ fails to be closed globally, one can still find a local potential on any contractible subdomain; thus the PHS form persists \emph{locally} but not necessarily globally.
\end{remark}

\subsection{Algebraic persistence under matrix perturbations}

Now consider perturbations to the structure/dissipation/port matrices. Replace $(J,R,G)$ by $(J+\Delta J,\,R+\Delta R,\,G+\Delta G)$. The perturbed dynamics become
\begin{equation}\label{eq:full_pert}
\dot x = \big((J+\Delta J)-(R+\Delta R)\big)\nabla H + (G+\Delta G)u + (J-R)\delta,
\end{equation}
where the extra term $(J-R)\delta$ arises if we also perturb the gradient as above; we will handle that by absorbing $\delta$ into a modified Hamiltonian when possible.

The algebraic constraints that define the PHS class are:
\[
J^\top = -J,\qquad R^\top = R \succeq 0.
\]
These constraints are preserved under small perturbations provided the perturbations respect skew/symmetry and positive semidefiniteness.

\begin{prop}[Persistence under admissible matrix perturbations]\label{prop:matrix}
Let $K\subset\mathcal{X}$ be compact. Suppose $\Delta J(x)$ is skew-symmetric for all $x\in K$, $\Delta R(x)$ is symmetric for all $x\in K$ and small in operator norm so that $R(x)+\Delta R(x)\succeq 0$ for all $x\in K$, and $\Delta G(x)$ is small. Then the perturbed matrices
\[
\widetilde J := J+\Delta J,\qquad \widetilde R := R+\Delta R,\qquad \widetilde G := G+\Delta G
\]
satisfy $\widetilde J^\top = -\widetilde J$ and $\widetilde R^\top = \widetilde R \succeq 0$ on $K$, hence the algebraic PHS constraints remain valid on $K$.
\end{prop}

\begin{proof}
Immediate from linearity of transpose and continuity of eigenvalues (Weyl inequalities). A sum of skew-symmetric matrices is skew-symmetric. For symmetric $R$ and symmetric $\Delta R$, eigenvalues vary continuously with the perturbation; choosing $\|\Delta R\|$ less than the smallest strictly positive gap of $R$ on $K$ (or ensuring no negative eigenvalue crosses zero) guarantees $\widetilde R\succeq 0$.
\end{proof}

\begin{thm}[Persistence of PHS form under natural perturbations]\label{thm:main}
Let \eqref{eq:phs_basic} be a PHS on a simply connected domain $\mathcal{X}$. Consider perturbations $\Delta J,\Delta R,\Delta G$ (matrix fields) and a 1-form perturbation $\delta$ so that the perturbed dynamics are
\[
\dot x = \big((J+\Delta J)-(R+\Delta R)\big)\big(\nabla H + \delta\big) + (G+\Delta G)u.
\]
Assume:
\begin{enumerate}
  \item $\Delta J$ is skew-symmetric on $\mathcal{X}$,
  \item $\Delta R$ is symmetric and small enough on compact subsets to keep $R+\Delta R\succeq 0$,
  \item $\delta$ is closed: $d\delta = 0$ on $\mathcal{X}$.
\end{enumerate}
Then there exists $\widetilde H = H + \varepsilon$ with $\mathrm{d}\varepsilon=\delta$ and
\[
\widetilde J := J+\Delta J,\qquad \widetilde R := R+\Delta R,\qquad \widetilde G := G+\Delta G
\]
such that the perturbed system can be written globally in PHS form
\[
\dot x = (\widetilde J - \widetilde R)\nabla \widetilde H + \widetilde G\, u.
\]
\end{thm}

\begin{proof}
By the Poincaré lemma (closed $\Rightarrow$ exact on simply connected $\mathcal{X}$) there exists $\varepsilon$ with $\mathrm{d}\varepsilon=\delta$. Set $\widetilde H=H+\varepsilon$. Substitute into the perturbed dynamics and collect matrix perturbations into $\widetilde J,\widetilde R,\widetilde G$. By Proposition \ref{prop:matrix} the algebraic PHS constraints hold for the perturbed matrices on compact subsets where $\Delta R$ is small, hence the system is again PHS with $(\widetilde J,\widetilde R,\widetilde H,\widetilde G)$.
\end{proof}

\begin{remark}
If $\delta$ is not closed globally, one may still obtain a local potential and hence a local PHS representation on contractible subsets. If $\delta$ is not expressible as an exact 1-form but can be decomposed as
\[
\delta = \delta_{\mathrm{grad}} + \delta_{\mathrm{ng}},
\]
where $\delta_{\mathrm{grad}}$ is exact and $\delta_{\mathrm{ng}}$ is of the form $(\Delta J - \Delta R)\nabla\widehat H$ for some $\Delta J=-\Delta J^\top$, $\Delta R=\Delta R^\top\succeq0$ and some $\widehat H$, then one can absorb $\delta_{\mathrm{grad}}$ into a modified Hamiltonian and interpret $\delta_{\mathrm{ng}}$ as further matrix perturbations. Such decompositions are problem-dependent and often possible for physically motivated perturbations.
\end{remark}

\subsection{More examples}

\subsubsection{Example 1: Mass--spring--damper (one degree of freedom)}

State $x = \begin{bmatrix} q \\ p \end{bmatrix}$ with $q$ position and $p$ momentum. Hamiltonian
\[
H(q,p) = \frac{p^2}{2m} + \frac{k q^2}{2},
\]
with $m>0$ mass and $k>0$ stiffness. Choose
\[
J = \begin{bmatrix} 0 & 1 \\ -1 & 0 \end{bmatrix},\qquad
R = \begin{bmatrix} 0 & 0 \\ 0 & c \end{bmatrix},\qquad
G=\begin{bmatrix}0\\1\end{bmatrix},
\]
where $c\ge 0$ is viscous damping and $u$ is an external force. The PHS dynamics are
\[
\dot x = (J-R)\nabla H + G u
= \begin{bmatrix} 0 & 1 \\ -1 & -c \end{bmatrix}
\begin{bmatrix} k q \\ \dfrac{p}{m} \end{bmatrix}
+ \begin{bmatrix}0\\1\end{bmatrix}u.
\]

Now perturb the mass $m\mapsto m+\Delta m$ with small $\Delta m$. The new Hamiltonian is
\[
\widetilde H(q,p)=\frac{p^2}{2(m+\Delta m)} + \frac{k q^2}{2}.
\]
Hence the perturbation to the dynamics modifies only $\nabla H$ (an exact perturbation), and by Proposition \ref{prop:poincare} we may absorb this into $\widetilde H$ while leaving $J,R,G$ unchanged. Thus the system remains PHS in the same base.

If instead we perturb the damping $c\mapsto c+\Delta c$, this changes $R$ to $R+\Delta R$ with
\[
\Delta R = \begin{bmatrix}0&0\\0&\Delta c\end{bmatrix},
\]
which is symmetric and (for $\Delta c\ge -c$) preserves positive semidefiniteness; hence PHS algebraic constraints persist (Proposition \ref{prop:matrix}).

\subsubsection{Example 2: Series RLC circuit}

Consider a series RLC circuit with state variables
\[
x = \begin{bmatrix} \phi \\ q \end{bmatrix},
\]
where $\phi$ is the flux linkage in the inductor and $q$ is the capacitor charge. Energies:
\[
H(\phi,q) = \frac{\phi^2}{2L} + \frac{q^2}{2C},
\]
with inductance $L>0$ and capacitance $C>0$. Take
\[
J = \begin{bmatrix} 0 & -1 \\ 1 & 0 \end{bmatrix},\qquad
R = \begin{bmatrix} r & 0 \\ 0 & 0 \end{bmatrix},\qquad
G=\begin{bmatrix}1\\0\end{bmatrix},
\]
where $r\ge 0$ is the resistance and the input $u$ is a voltage source. The PHS equations are
\[
\dot x = (J-R)\nabla H + G u
= \begin{bmatrix} 0 & -1 \\ 1 & -r \end{bmatrix}
\begin{bmatrix} \dfrac{\phi}{L} \\ \dfrac{q}{C} \end{bmatrix}
+ \begin{bmatrix}1\\0\end{bmatrix}u.
\]

Perturbing $L$ or $C$ (e.g.\ $L\mapsto L+\Delta L$) modifies $H$ to
\[
\widetilde H(\phi,q)=\frac{\phi^2}{2(L+\Delta L)}+\frac{q^2}{2C},
\]
which is an exact change to $\nabla H$ and hence absorbed into $\widetilde H$. Perturbing $r$ modifies $R$ directly in a symmetric way; small changes preserve $R\succeq 0$. Thus practical component tolerances correspond to the admissible perturbation classes and the circuit remains representable in the same PHS base.

\subsubsection{Damped, forced wave equation}

Consider the damped wave equation
\[
\frac{\partial^2 u}{\partial t^2} - c^2 \nabla^2 u + d\,\frac{\partial u}{\partial t}
= f(x,t),
\qquad x\in\Omega\subset\mathbb{R}^n.
\]

Introduce the state
\[
x = 
\begin{bmatrix}
u \\ p
\end{bmatrix},
\qquad
p = \frac{\partial u}{\partial t},
\]
where $u$ is displacement and $p$ is momentum density.  
The Hamiltonian functional (total stored energy) is
\[
H(u,p) = \int_\Omega 
\left(
\frac{p^2}{2} + \frac{c^2}{2}\|\nabla u\|^2
\right)\,dx.
\]

Define the variational gradients
\[
\nabla H =
\begin{bmatrix}
-\;c^2 \nabla^2 u \\[4pt]
p
\end{bmatrix}.
\]

Choose the port-Hamiltonian structure matrices
\[
J = 
\begin{bmatrix}
0 & 1 \\[4pt]
-1 & 0
\end{bmatrix},
\qquad
R = 
\begin{bmatrix}
0 & 0 \\[4pt]
0 & d
\end{bmatrix},
\qquad
G =
\begin{bmatrix}
0 \\[4pt]
1
\end{bmatrix},
\]
where $d\ge 0$ is the damping coefficient and $u_{\mathrm{in}}=f(x,t)$ is an external force density.

The port-Hamiltonian dynamics are
\[
\dot{x}
= (J-R)\nabla H + G\,u_{\mathrm{in}}
=
\begin{bmatrix}
0 & 1 \\[4pt]
-1 & -d
\end{bmatrix}
\begin{bmatrix}
-\;c^2 \nabla^2 u \\[4pt]
p
\end{bmatrix}
+
\begin{bmatrix}
0 \\[4pt]
1
\end{bmatrix}
f(x,t).
\]

Expanding the components gives
\[
\dot{u} = p,
\qquad
\dot{p} = c^2 \nabla^2 u - d\,p + f(x,t),
\]
which recovers the original PDE.

A perturbation in the wave speed $c\mapsto c+\Delta c$ modifies only the Hamiltonian (an exact perturbation), leaving $J,R,G$ unchanged; hence the system remains in the same port-Hamiltonian base.  
In contrast, perturbing the damping $d\mapsto d+\Delta d$ alters the dissipation matrix to
\[
R+\Delta R
=
\begin{bmatrix}
0 & 0 \\[4pt]
0 & d+\Delta d
\end{bmatrix},
\]
which preserves symmetry and, for $\Delta d\ge -d$, maintains positive semidefiniteness; therefore the port-Hamiltonian algebraic constraints are still satisfied.

\section{Experimental Details}

Here we provide further experimental details and supplemental figures for the experiments in Section 3 of the main text.

\subsection{2D damped Kuramoto–Sivashinsky equation system}

In this setting, our data simulation model is the 2D version of damped Kuramoto-Sivashinsky equation (KS equation), which is given by

\[
\bf{u}_t + \frac12 |\nabla \bf{u}|^2 + \nabla^2 \bf{u} + \nu\nabla^4 \bf{u} = 0
\]

while the real physics is a damped version of the KS equation (DKS) \cite{dehghan2019two}, given by
\begin{equation}
\bf{u}_t^\prime  + \frac12 |\nabla \bf{u}^\prime|^2 + \nabla^2 \bf{u}^\prime + \nu\nabla^4 \bf{u}^\prime + (\bf{v}\cdot \nabla)\bf{u}^\prime - \gamma \bf{u}^\prime= 0
\label{eq:2DKS_real}
\end{equation}
with damping force $(\bf{v}\cdot \nabla)\bf{u}^\prime - \gamma \bf{u}^\prime$. 

Still, for simplicity we set the fourth order parameter to be $\nu = 1$. Figure~\ref{fig:ks_2} exhibits its behavior on a 2D map of range $[0,64]\times[0,64]$.

\begin{figure}[t]
\centering
\begin{overpic}[scale=0.3]{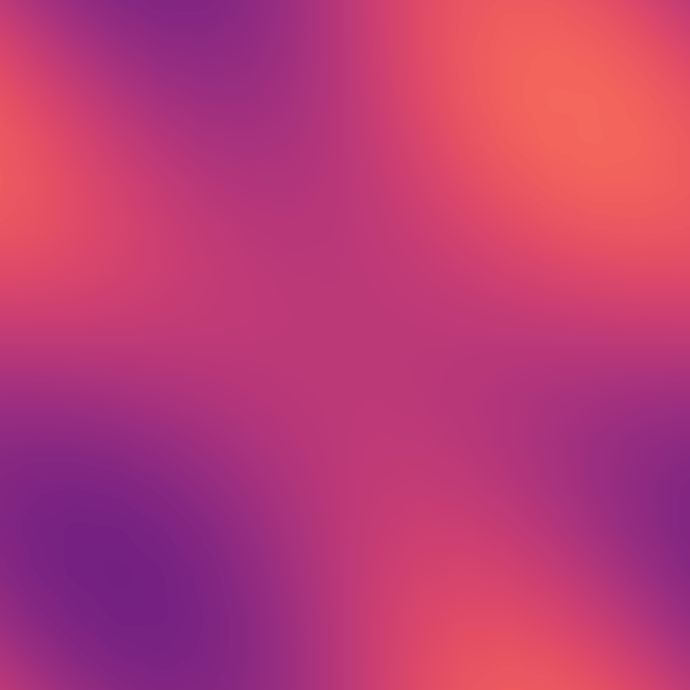}
\put(40,105){$t=0.0$}
\end{overpic}
\begin{overpic}
[scale=0.3]{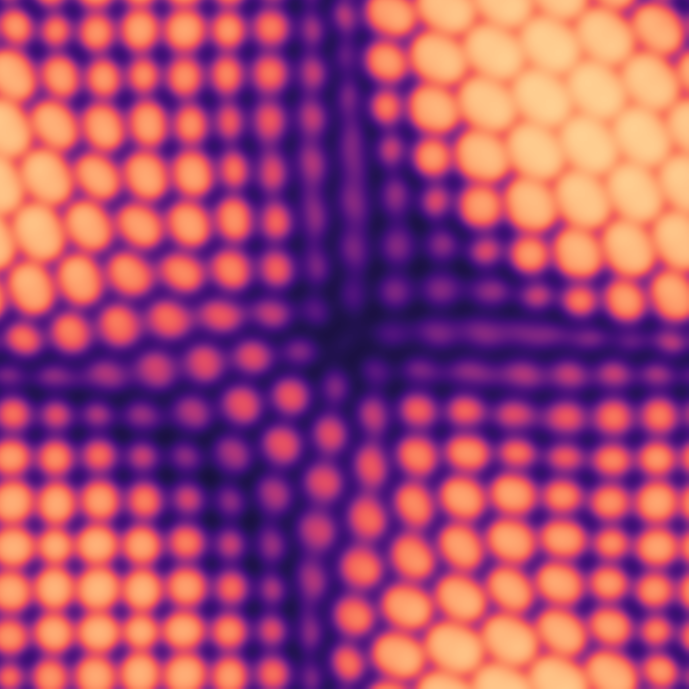}
\put(40,105){$t=5.0$}
\end{overpic}
\begin{overpic}
[scale=0.3]{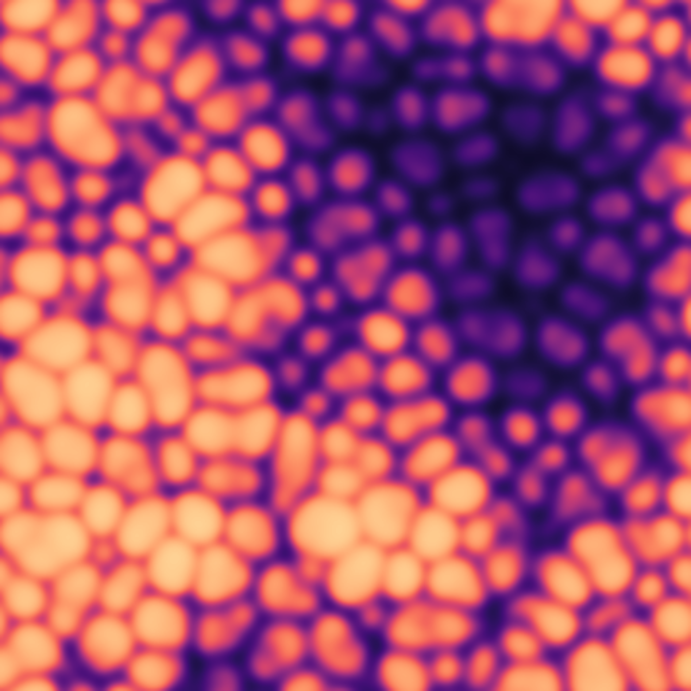}
\put(40,105){$t=10.0$}
\end{overpic}
\begin{overpic}
[scale=0.3]{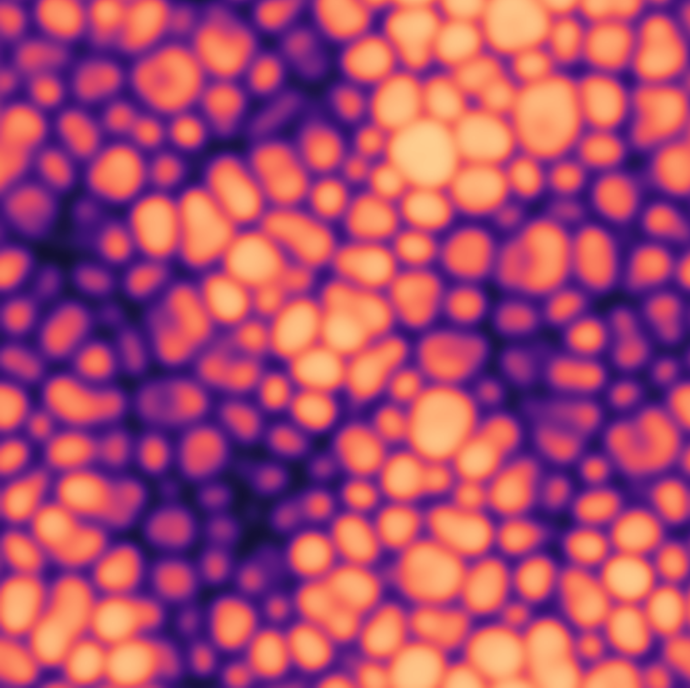}
\put(40,105){$t=15.0$}
\end{overpic}

\caption{The behavior of 2D Kuramoto-Sivashinsky equation without damping. \label{fig:ks_2}}
\end{figure}

\subsection{2D Kolmogorov flow}

Figure~\ref{fig:kf_2} shows the results of the 2D damped Kolmogorov flow in earlier time steps. In contrast to the fast convergence from simpler examples, the DA-SHRED model did struggle a bit with perturbations on the 2D Kolmogorov flow.

At the initial stages of the data assimilation process, the DA-SHRED model exhibited poor performance, as shown in Figure 9, likely due to an incomplete alignment between the model’s internal dynamics and the observed data. During this early period ($t<20$), the system’s state estimates were still heavily influenced by prior uncertainties and insufficiently informed, leading to large discrepancies between DA-SHRED restorations and ground truth. However, afterwards, the model began to stabilize and progressively converge toward the real physics, suggesting that the assimilation updates had accumulated enough information to reduce model bias and variance, allowing the DA-SHRED framework to achieve a consistent alignment between its restored and unobserved real state spaces.

\begin{figure}[H]
\centering
\begin{minipage}{0.5\linewidth}
  \begin{subfigure}{\linewidth}
    \begin{overpic}[width=\linewidth]{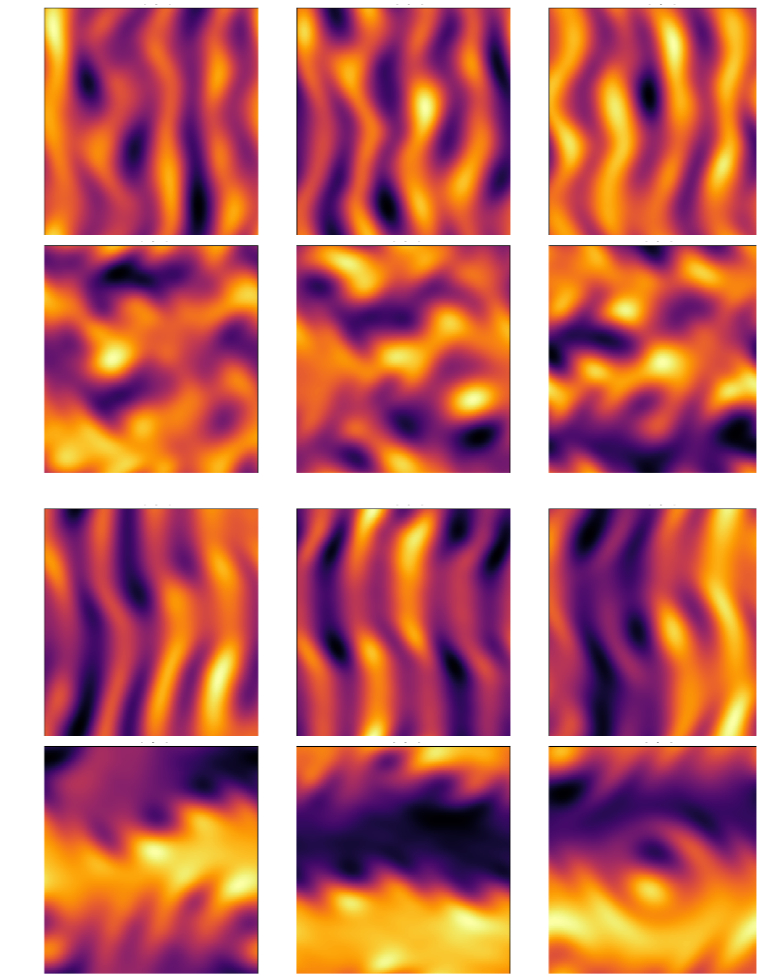}
    \put(8,101){simulation}
    \put(32,101){real physics}
    \put(57,101){DA-SHRED}
    \put(34,13){\color{white} $\bf t\!=\!20.0$}
    \put(34,64.5){\color{white} $\bf t\!=\!15.0$}
    \put(33,-3){\color{black} ${{L}} = \{(\omega)\}$}
    \put(53,-3){\color{black} ${{L}'} = \{(\omega, {\color{gray}{\omega\nabla^2 \omega,\omega^3}})\}$}
    \end{overpic}
    \caption*{}\label{subfig:key-a}
  \end{subfigure}
  \end{minipage}
\begin{minipage}{0.45\linewidth}
    \begin{subfigure}{\linewidth}
      \vstretch{1.1}{\includegraphics[width=\linewidth]{figs/RMSE_kf_p1.png}}
      \caption*{}\label{subfig:key-b}
    \end{subfigure}\hfill

    \medskip
    \begin{subfigure}{\linewidth}
      \vstretch{1.1}{\includegraphics[width=\linewidth]{figs/RMSE_kf_p2.png}}
      \caption*{}\label{subfig:key-c}
    \end{subfigure}\hfill
  \end{minipage}
\caption{The result of 2D damped Kolmogorov flow. Figures on the same row are taken at the same timestep. The left column represents the undamped simulation model (2D Kolmogorov flow without damping); the middle column represents unknown real physics (2D Kolmogorov flow with linear damping); and the right column represents state space restored by DA-SHRED.
\label{fig:kf_2}}
\end{figure}

\subsection{1D rotating detonation engines}

Figure~\ref{fig:rde_2} shows the results of the 1D rotating detonation engines in later time steps after the meet of detonation fronts.

In the one-dimensional rotating detonation engine (RDE) framework, simplified numerical models typically represent the detonation front as a sharp and well-defined discontinuity. However, in practical RDE configurations, the injector dynamics --- including spatial non-uniformities in fuel injection, local turbulence, and incomplete mixing --- introduce complex flow features that effectively broaden the detonation front. This physical broadening arises from localized variations in pressure, temperature, and equivalence ratio near the injection region, and leads to an extended reaction zone rather than an idealized, infinitesimally thin front. In the DA-SHRED experiments, this behavior was clearly observed in both figure 6 and figure 10: the detonation front exhibited a wider structure compared to the corresponding simplified 1D simulations. This observation gives further hint on the critical role of injector-induced flow inhomogeneities in determining the effective detonation structure and underscores the necessity of incorporating realistic injector and mixing physics in data assimilation frameworks for accurate RDE modeling.

\begin{figure}[H]
  \begin{minipage}{0.38\linewidth}
    \begin{subfigure}{\linewidth}
    \begin{overpic}[width=\linewidth]{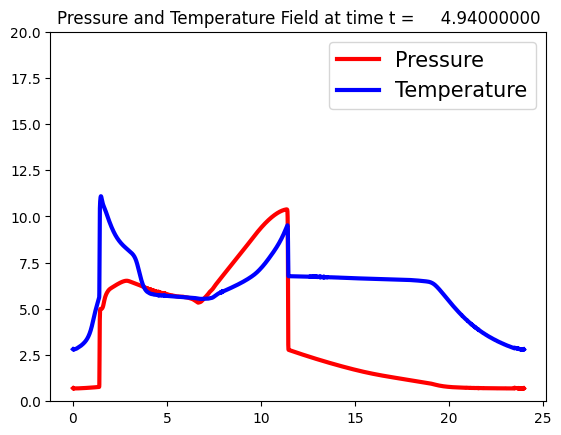}
    \put(34,81){simulation}
    \end{overpic}
      \caption*{}\label{subfig:key-b}
    \end{subfigure}\hfill
  \end{minipage}
      \begin{minipage}{0.6\linewidth}
  \begin{subfigure}{\linewidth}
    \begin{subfigure}{\linewidth}
    \begin{overpic}[width=0.5\linewidth]{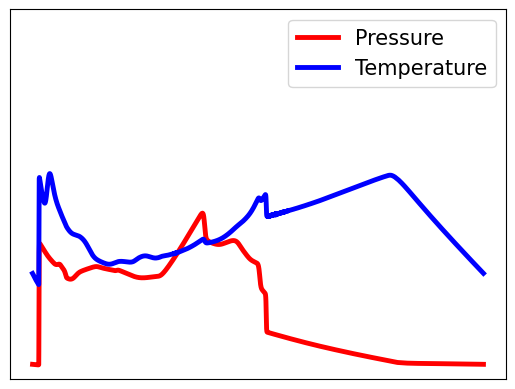}
    \put(34,79){real physics}
    \put(30,-7){\color{black} ${{L}'} = \{({\bf u},{\bf u}^3)\}$}
    \end{overpic}
    \begin{overpic}[width=0.5\linewidth]{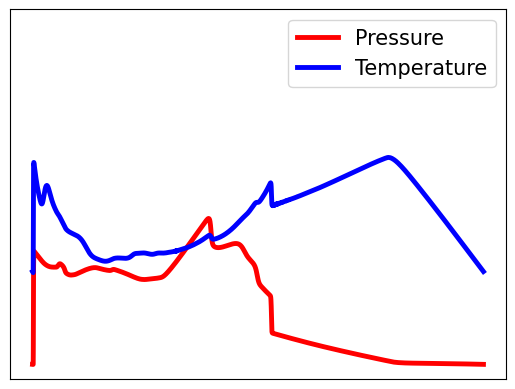}
    \put(34,79){DA-SHRED}
    \put(24,-7){\color{black} ${{L}'} = \{({\bf u},{\bf u}^3, {\color{gray}{\lambda^2 {\bf u}}})\}$}
    \end{overpic}
      \caption*{}\label{subfig:key-b}
    \end{subfigure}\hfill

    \medskip
    \begin{subfigure}{\linewidth}
      \begin{overpic}[width=0.5\linewidth]{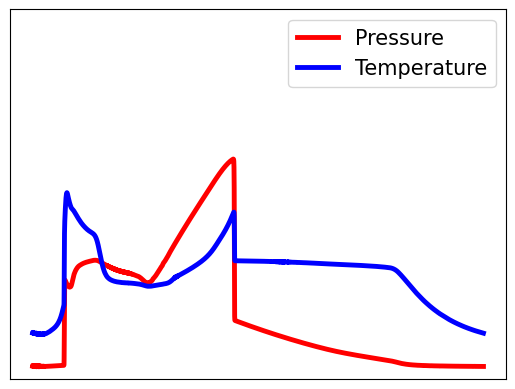}
    \put(36,-7){\color{black} ${{L}'} = \{({\bf u})\}$}
    \end{overpic}
      \begin{overpic}[width=0.5\linewidth]{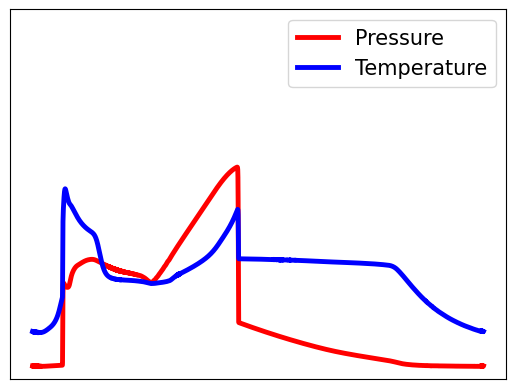}
    \put(36,-7){\color{black} ${{L}'} = \{({\bf u})\}$}
    \end{overpic}
      \caption*{}\label{subfig:key-c}
    \end{subfigure}\hfill
  \end{subfigure}
  \end{minipage}
  \caption{The DA-SHRED reconstruction of 1D damped RDE model at a later timestep beyond that displayed in Figure~\ref{fig:rde}. Similarly, one simulation
model is deployed on two different real physics settings, and the rightmost column represents the corresponding
two DA-SHRED results towards convergence.
  \label{fig:rde_2}}
\end{figure}

\end{document}